\documentclass{article}

\usepackage[utf8]{inputenc} 
\usepackage[T1]{fontenc}    
\usepackage{hyperref}       
\usepackage{url}            
\usepackage{booktabs}       
\usepackage{amsfonts}       
\usepackage{nicefrac}       
\usepackage{microtype}      
\usepackage{xcolor}         
\usepackage{algorithm}      
\usepackage{algpseudocode}  
\usepackage{amssymb}        
\usepackage{diagbox}      
\usepackage[normalem]{ulem} 

\usepackage{amsmath}
\usepackage{amsthm}
\usepackage{empheq}

\usepackage{mathtools}

\usepackage{wrapfig,soul}
\usepackage{enumitem}
\usepackage{xcolor}
\usepackage{grffile}
\usepackage{thm-restate}
\usepackage{natbib}
\usepackage{lipsum}
\usepackage{titlesec}
\usepackage{caption} 
\usepackage{subcaption} 
\usepackage{multirow} 
\usepackage{makecell} 
\usepackage{graphics} 
\usepackage{pifont}
\newcommand{\cmark}{\ding{51}}%
\newcommand{\xmark}{\ding{55}}%

\newtheorem{Lemma}{Lemma}

\newtheorem{Theorem}{Theorem}
\newtheorem{Assumption}{Assumption}
\newtheorem{Remark}{Remark}

\newcommand{\E}{\mathbb{E}}

\newcommand{\red}[1]{{\color{red} #1}}
\newcommand{\blue}[1]{{\color{blue} #1}}

\newcommand{\gray}[1]{{\color{gray} #1}}
\newcommand{\yellow}[1]{{\color{yellow} #1}}

\titlespacing\subsection{0pt}{12pt plus 4pt minus 2pt}{0pt plus 2pt minus 2pt}
\renewcommand\footnotemark{}
 
\usepackage{fullpage}
\usepackage{times}
\usepackage{url}
\usepackage{natbib}
\usepackage{amsthm,amsfonts,amsmath,amssymb,epsfig,color,float,graphicx,verbatim}
\usepackage{enumitem}
\usepackage{algorithm}
\usepackage{algpseudocode}

\usepackage{wrapfig}
\usepackage{hyperref}
\usepackage{multirow}
\usepackage{tabularx}
\usepackage{prettyref}
\usepackage{authblk}

\hypersetup{
    colorlinks=true,
    linkcolor=blue,
    filecolor=blue,      
    urlcolor=blue,
    citecolor= blue
    }

\usepackage[verbose=true,letterpaper]{geometry}
\AtBeginDocument{
  \newgeometry{
    textheight=8.8in,
    textwidth=6in,
    top=1in,
    headheight=12pt,
    headsep=25pt,
    footskip=30pt
  }
}

\renewenvironment{abstract}%
{%
  \centerline%
  {\large\bf Abstract}%
  \vspace{0.5ex}%
  \begin{quote}%
}
{
  \par%
  \end{quote}%
  \vskip 1ex%
}

\title{Mitigating Gradient Bias in Multi-objective Learning:\\ A Provably Convergent Stochastic Approach}
\author[*]{Heshan Fernnado}
\author[*]{Han Shen}
\author[$\dag$]{Miao Liu}
\author[$\dag$]{Subhajit Chaudhury}
\author[$\dag$]{Keerthiram Murugesan}
\author[*]{Tianyi Chen\thanks{This work was supported by the Rensselaer-IBM AI Research Collaboration (\url{http://airc.rpi.edu}), part of the IBM AI Horizons Network (\url{http://ibm.biz/AIHorizons}). Correspondence to: Heshan Fernnado (\href{mailto:fernah@rpi.edu}{fernah@rpi.edu})}}
\affil[*]{Rensselaer Polytechnic Institute, New York, United States}
\affil[$\dag$]{IBM Research, New York, United States}

\date{}

\begin{document}

\maketitle

\begin{abstract}
Machine learning problems with multiple objective functions appear either in learning with multiple criteria where learning has to make a trade-off between multiple performance metrics such as fairness, safety and accuracy; or, in multi-task learning where multiple tasks are optimized jointly, sharing inductive bias between them. This problems are often tackled by the multi-objective optimization framework. However, existing stochastic multi-objective gradient methods and its variants (e.g., MGDA, PCGrad, CAGrad, etc.) all adopt a biased noisy gradient direction, which leads to degraded empirical performance. 
To this end, we develop a stochastic {\sf M}ulti-objective gradient {\sf C}orrection ({\sf MoCo}) method for multi-objective optimization. The unique feature of our method is that it can guarantee convergence without increasing the batch size even in the nonconvex setting. Simulations on multi-task supervised and reinforcement learning demonstrate the effectiveness of our method relative to   state-of-the-art methods. 
\end{abstract}

\section{Introduction}\label{sec:introduction}
 
        Multi-objective optimization (MOO) involves optimizing multiple, potentially conflicting objectives simultaneously. Recently, MOO has gained attention in various application settings such as optimizing hydrocarbon production \citep{you2020development}, tissue engineering \citep{shi2019multi}, safe reinforecement learning \citep{thomas2021multi}, and training neural networks for multiple tasks \citep{sener2018multi}. We consider the stochastic MOO problem as 
            \begin{equation}
                \min\limits_{x\in \mathbb{R}^d} F(x) \coloneqq \left(\E_\xi[f_1(x, \xi)], \E_\xi[f_2(x, \xi)], \dots, \E_\xi[f_M(x, \xi)] \right) \label{eq:prob-form}
            \end{equation} 
        where $d$ is the dimension of the parameter $x$, and $f_m:\mathbb{R}^d\mapsto\mathbb{R}$ with $f_m(x) := \E_\xi[f_m(x, \xi)]$ for $m\in[M]$.  Here we denote $[M] := \{1, 2, \dots, M\}$ and denote $\xi$ as a random variable. In this setting, we are interested in optimizing all of the objective functions simultaneously without sacrificing any individual objective. Since we cannot always hope to find a common variable $x$ that achieves optima for all functions simultaneously, a natural solution instead is to find the so-termed \emph{Pareto stationary point} $x$ that cannot be further improved for all objectives without sacrificing some objectives. In this context, multiple gradient descent algorithm (MGDA) has been developed for achieving this goal \citep{Desideri2012mgda}. The idea of MGDA is to iteratively update the variable $x$ via a common descent direction for all the objectives through a time-varying convex combination of gradients from individual objectives. 
        Recently, different MGDA-based MOO algorithms have been proposed, especially for solving multi-task learning (MTL) problems \citep{sener2018multi, chen2018gradnorm, yu2020gradient, liu2021conflict}. 
        
        While the deterministic MGDA algorithm and its variants are well understood in literature, only little theoretical study has been taken on its stochastic counterpart. Recently, \citep{liu2021stochastic} has introduced the stochastic multi-gradient (SMG) method as a stochastic counterpart of MGDA (see Section \ref{sec:smg} for details). 
        To establish convergence however, \citep{liu2021stochastic} requires the strong assumption on the fast decaying first moment of the gradient, which was enforced by linearly growing the batch size. 
        While this allows for analysis of multi-objective optimization in stochastic setting, it may not be true for many MTL tasks in practice. Furthermore, the analysis in \citep{liu2021stochastic} cannot cover the important setting with non-convex multiple objectives, which is prevalent in challenging MTL tasks. This leads us to a natural question: 
        \begin{center}
                {\em Can we design a stochastic MOO algorithm that provably converges to a Pareto stationary point without growing batch size and also in the nonconvex setting?}
        \end{center}

        \begin{figure}[t]
             \centering
             \begin{subfigure}[b]{0.27\textwidth}
                 \centering
                 \includegraphics[width=\textwidth]{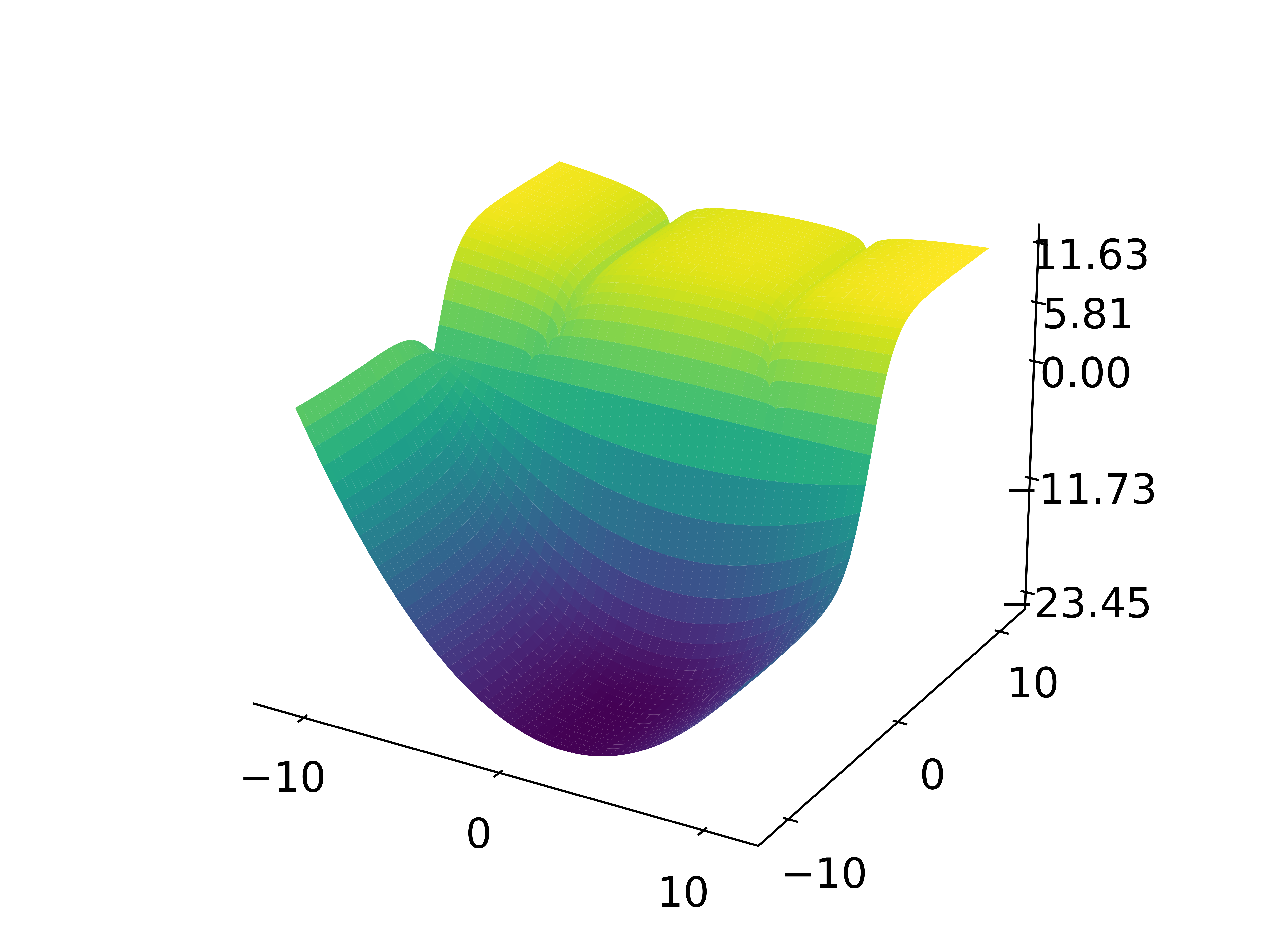}
                 \caption{Mean objective}
                 \label{fig:grad-3d-obj}
             \end{subfigure}
             \centering
             \begin{subfigure}[b]{0.25\textwidth}
                 \centering
                 \includegraphics[width=\textwidth]{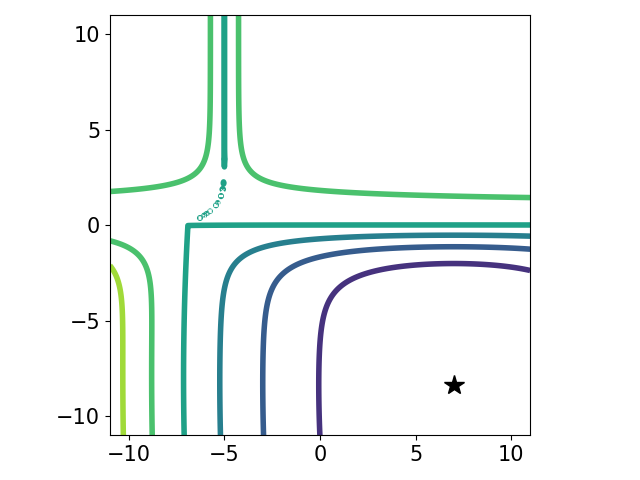}
                 \caption{Objective 1}
                 \label{fig:grad-task-1}
             \end{subfigure}
             \centering
             \begin{subfigure}[b]{0.25\textwidth}
                 \centering
                 \includegraphics[width=\textwidth]{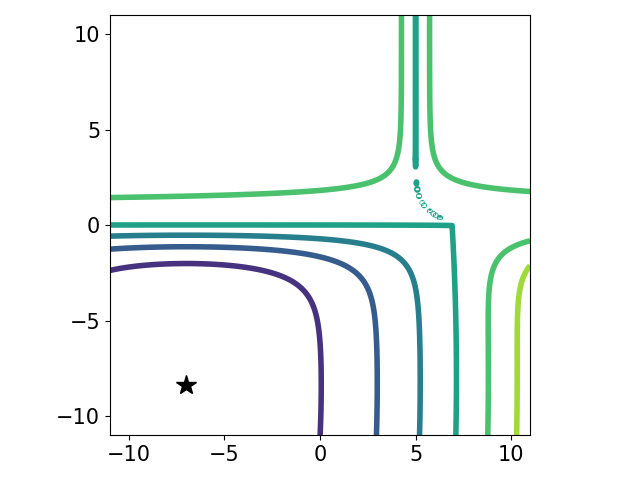}
                 \caption{Objective 2}
                 \label{fig:grad-task-2}
             \end{subfigure}\\[5pt]
             \centering
             \begin{subfigure}[b]{0.19\textwidth}
                 \centering
                 \includegraphics[width=\textwidth]{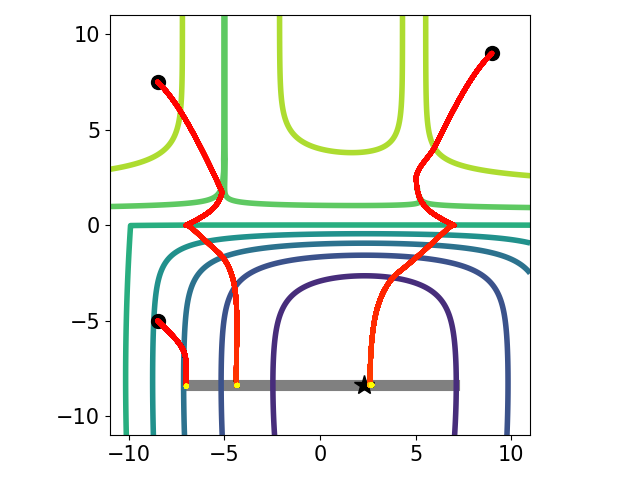}
                 \caption{MGDA}
                 \label{fig:toy-mgda}
             \end{subfigure}
             \hfill
             \begin{subfigure}[b]{0.19\textwidth}
                 \centering
                 \includegraphics[width=\textwidth]{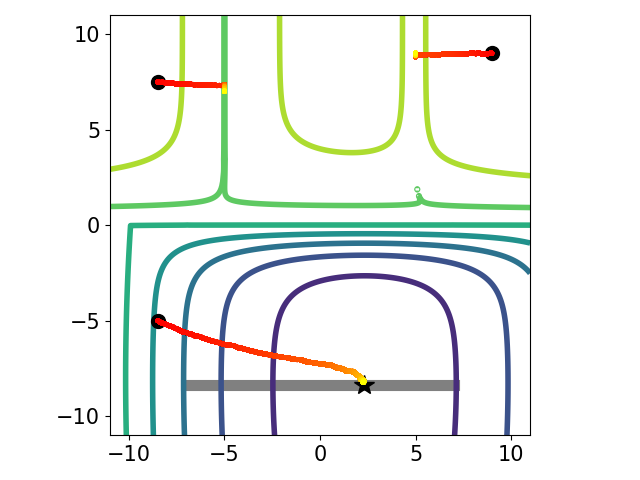}
                 \caption{SMG}
                 \label{fig:toy-smgd}
             \end{subfigure}
             \hfill
             \begin{subfigure}[b]{0.19\textwidth}
                 \centering
                 \includegraphics[width=\textwidth]{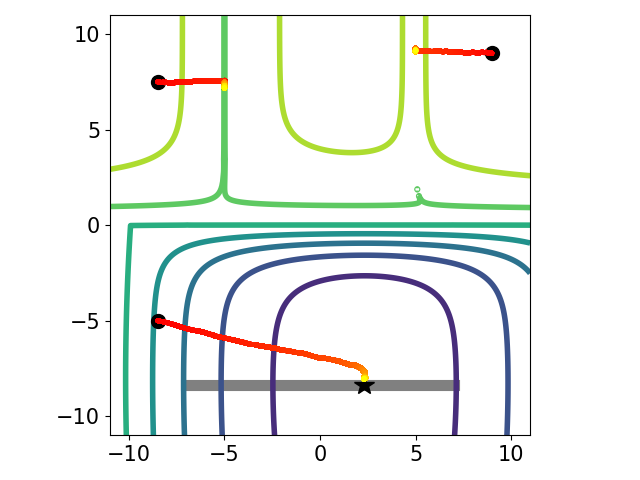}
                 \caption{PCGrad}
                 \label{fig:toy-pcgrad}
             \end{subfigure}
             \hfill
             \begin{subfigure}[b]{0.19\textwidth}
                 \centering
                 \includegraphics[width=\textwidth]{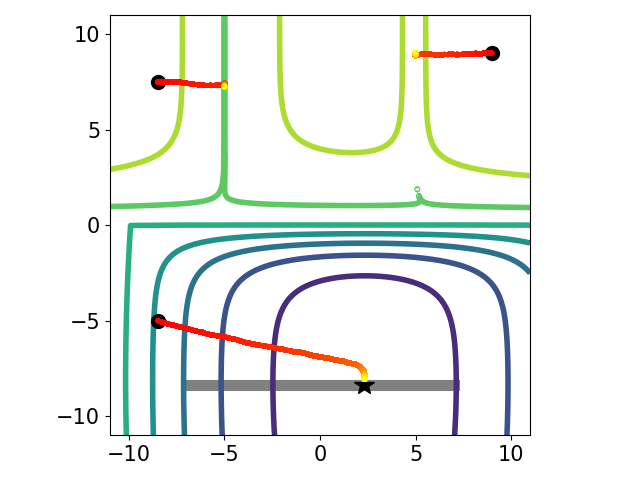}
                 \caption{CAGrad}
                 \label{fig:toy-cagrad}
             \end{subfigure}
             \hfill
             \begin{subfigure}[b]{0.19\textwidth}
                 \centering
                 \includegraphics[width=\textwidth]{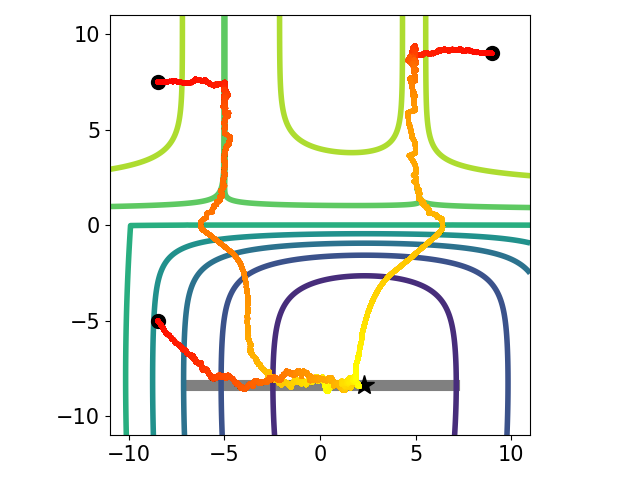}
                 \caption{MoCo (ours)}
                 \label{fig:toy-tracking}
             \end{subfigure}
                \caption{A toy example from \citep{liu2021conflict} with two objective (Figures \ref{fig:grad-task-1} and \ref{fig:grad-task-2}) to show the impact of gradient bias. We use the mean objective as a reference when plotting the trajectories corresponding to each initialization (3 initializations in total). The starting points of the trajectories are denoted by a {\bf black} $\bullet$, and the trajectories are shown fading from \red{\bf red} (start) to \yellow{\bf yellow} (end). The Pareto front is given by the \gray{\bf gray} bar, and the {\bf black} $\star$ denotes the point in the Pareto front corresponding to equal weights to each objective. We implement recent MOO algorithms such as SMG \citep{liu2021stochastic}, PCGrad \citep{yu2020gradient}, and CAGrad \citep{liu2021conflict}, and MGDA \citep{Desideri2012mgda} alongside our method. Except for MGDA (Figure \ref{fig:toy-mgda}) all the other algorithms only have access to gradients of each objective with added zero mean Gaussian noise. It can be observed that SMG, CAGrad, and PCGrad fail to find the Pareto front in some initializations.}
                \label{fig:toy-comp}
        \end{figure}

\begin{table}[tb]
    \small
    \centering
    \begin{tabular}{c| c  |c |c |c|c|c}
      \hline\hline
         Method & \makecell{Batch size} & \makecell{Non-convex}   & \makecell{Assuming Lipschitz \\ continuity of  $\lambda^*(x)$} & \makecell{Bounded \\ functions} & \makecell{Allow biased \\  gradient} & \makecell{Sample \\ complexity}\\
         \hline\hline
         \makecell{SMG \\ \!\!\!\citep{liu2021stochastic}\!\!\!} & $\mathcal{O}\left(\epsilon^{-2}\right)$ & \xmark &   \cmark & \xmark & \xmark & $\mathcal{O}\left(\epsilon^{-4}\right)$\\
         \hline
         \makecell{MoCo \\ (Theorem \ref{theorem:upper})} & $\mathcal{O}(1)$ &   \cmark & \xmark & \xmark & \cmark & $\mathcal{O}\left(\epsilon^{-12}\right)$\\
         \hline
         \makecell{MoCo \\ (Theorem \ref{theorem:upper-fbound})} & $\mathcal{O}(1)$ &  \cmark & \xmark & \cmark & \cmark & $\mathcal{O}\left(\epsilon^{-2.5}\right)$\\
         \hline
         \makecell{MoCo \\ (Theorem \ref{theorem:upper-fbound-stronger})} & $\mathcal{O}(1)$ &   \cmark & \xmark & \cmark & \cmark & $\mathcal{O}\left(\epsilon^{-2}\right)$ \\
        \hline\hline
    \end{tabular}
    \caption{Comparison of MoCo with prior work on gradient based stochastic MOO, stochastic multi-gradient method (SMG). Here the ``Batch size'' column represents the number of samples used at each (outer level) iteration, ``Non-convex'' column denotes whether the analysis is valid for non-convex functions, ``Lipschitz of $\lambda^*(x)$'' column denotes whether Lipschitz continuity of $\lambda^*(x)$ (see Remark \ref{remark:lambda}) with respect to $x$ was assumed,  ``Bounded Function'' column denotes whether boundedness of function values was assumed, ``Biased stochastic gradient'' column denote whether the analysis allows for bias in the stochastic gradient of the functions, and ``Sample complexity'' column provides the (outer level) sample complexity of the corresponding method.}
    \label{tab:result-comp-summ}
\end{table}

        \textbf{Our contributions.} 
        In this paper, we answer this question affirmatively by providing the first stochastic MOO algorithm that provably converges to a Pareto stationary point without growing batch size. Specifically, we make the following major contributions:
        \begin{enumerate}
            \item [\bf C1)] \textbf{(Asymptotically unbiased multi-gradient).} We introduce a new method for MOO that we call stochastic Multi-objective gradient with Correction (MoCo) method. MoCo is a simple algorithm that addresses the convergence issues of stochastic MGDA and provably converges to a Pareto stationary point under several stochastic MOO settings. We use a toy example in Figure \ref{fig:toy-comp} to demonstrate the empirical benefit of our method. In this example, 
           MoCo is able to reach the Pareto front from all initializations, while other MOO algorithms such as SMG, CAGrad, and PCGrad fail to find the Pareto front due to using biased multi-gradient.
       
            \item [\bf C2)] \textbf{(Unified non-asymptotic analysis).} We generalize our MoCo method to the case where the individual objective function has a nested structure and thus obtaining unbiased stochastic gradients is costly. We provide a unified convergence analysis of the nested MoCo algorithm in smooth non-convex MOO setting. To our best knowledge, this is the first analysis of smooth non-convex stochastic gradient based MOO. A comparison of our results with the closest prior work can be found in Table \ref{tab:result-comp-summ}. 
            \item [\bf C3)] \textbf{(Experiments on MTL applications).} We provide an empirical evaluation of our method with existing state of the art MTL algorithms in supervised learning and reinforcement learning settings, and show that our method can outperform prior methods such as stochastic MGDA, PCGrad, CAGrad, GradDrop. For evaluation  in the  multi-task supervised learning setting we use cityscapes and NYU-v2 datasets.
        \end{enumerate}
        
        \textbf{Organization.} 
        In Section \ref{sec:background}, we provide the concepts and background related to our proposed method. In Section \ref{sec:method}, we introduce our algorithm and provide the convergence analysis. In section \ref{sec:related-work}, we discuss the prior work. In Section \ref{sec:experiments} we provide an empirical evaluation of our method using multi-task learning benchmarks, followed by the conclusions in Section \ref{sec:conclusion}.

\section{Background}\label{sec:background}

    In this section, we introduce the concepts of Pareto optimality and Pareto stationarity, and then discuss the MGDA and its existing stochastic counterpart. We then motivate our proposed method by elaborating the issues in the SMG method for stochastic MOO.

        \subsection{Pareto optimality and Pareto stationarity}
        We are interested in finding the points which can not be improved simultaneously for the objective functions, leading to the notion of {\em Pareto optimality}. 
         Consider two points $x, x'\in\mathbb{R}^d$. We say that $x$ dominates $x'$ if $f_m(x)\leq f_m(x') $ for all $ m \in [M] $, and $ F(x) \neq F(x') $. If a point $x^*\in\mathbb{R}^d$ is not dominated by any $x\in\mathbb{R}^d$, we say $x^*$ is Pareto optimal. 
         The collection of all the Pareto optimal points is called as the Pareto set. The collection of vector objective values $F(x^*)$ for all   Pareto optimal points $x^*$ is called as the Pareto front.
         
         Akin to the single objective case, a necessary condition for Pareto optimality is \emph{Pareto stationarity}. If $x$ is a Pareto stationary point, then there is no common descent direction for all $f_m(x)$ at $x$. Formally, $x$ is a called a Pareto stationary point if $\text{range}( \nabla F(x)^\top ) \cap (-\mathbb{R}^M_{++}) = \emptyset $ where $\nabla F(x)\in \mathbb{R}^{d \times M}$ is the Jacobian of $F(x)$, i.e. $\nabla F(x) \coloneqq \left( \nabla f_1(x), \nabla f_2(x), \dots, \nabla f_M(x) \right)$, and $\mathbb{R}^M_{++}$ is the positive orthant cone. When all $f_m(x)$ are strongly convex, a Pareto stationary point is also Pareto optimal.  

        \subsection{Multiple Gradient Descent Algorithm} \label{sec:mgda}
    The MGDA algorithm has been proposed in \citep{Desideri2012mgda} that can converge to a Pareto stationary point of $F(x)$.
       MGDA achieves this by seeking a convex combination of individual gradients $\nabla f_m(x)$ (also known as the multi-gradient), given by $d(x) = \sum_{m=1}^M \lambda_m^*(x) \nabla f_m(x)$ where the weights $\lambda^*(x) \coloneqq (\lambda_1^*(x),...,\lambda_M^*(x))^\top$ are found by solving the following sub-problem:
        \begin{align}
            \lambda^*(x) \in \arg\min_{\lambda} ~ \|\nabla F(x) \lambda\|^2~~~{\rm s.~t.}~~~\lambda\in\Delta^M \coloneqq \{\lambda \in \mathbb{R}^M~|~ \mathbf{1}^\top \lambda = 1,~\lambda \geq 0\}. \label{eq:lambd-det-prob}
        \end{align}
        With this multi-gradient $d(x)$, the $k$th iteration of MGDA is given by 
\begin{equation}\label{eq:mgda}
    x_{k+1} = x_k - \alpha_k d(x_k)~~~{\rm with}~~~d(x_k) = \sum_{m=1}^M \lambda_m^*(x_k) \nabla f_m(x_k)
\end{equation}
        where $\alpha_k$ is the learning rate. It can be shown that the MGDA algorithm optimizes all  objectives simultaneously following the  direction $-d(x)$ whenever $x$ is not a Pareto stationary point, and will terminate once it reaches a Pareto stationary point \citep{fliege2019complexity}.
        
        However, in many real world applications we either do not have access to the true gradient of functions $f_m$ or obtaining the true gradients is prohibitively expensive in terms of computation. This leads us to a possible stochastic counterpart of MGDA, which is discussed next.

\subsection{Stochastic multi-objective gradient and its brittleness} \label{sec:smg} 

        The stochastic counterpart of MGDA, referred to as SMG algorithm, has been studied in \citep{liu2021stochastic}. In SMG algorithm, the stochastic multi-gradient is obtained by replacing the true gradients $\nabla f_m(x)$ in (\ref{eq:lambd-det-prob}) with their stochastic approximations $\nabla f_m(x, \xi)$, where $\nabla f_m(x) = \mathbb{E}_\xi \left[\nabla f_m(x, \xi)\right]$. Specifically, the stochastic multi-gradient is given by 
        \begin{equation}
     g(x, \xi) = \sum_{m=1}^{M} \lambda^g_m(x, \xi) \nabla f_m(x, \xi)~~~{\rm with}~~~        \lambda^g(x, \xi) \in \arg\!\min\limits_{ \lambda\in\Delta^M} \left\Vert \sum_{m=1}^{M} \lambda_m\nabla f_m(x, \xi) \right\Vert^2. \label{eq:lambd-stoch-prob}
        \end{equation}
        
        While this change of the subproblem facilitates use of stochastic gradients in place of deterministic gradients, it raise issues in the biasedness in the stochastic multi-gradient calculated in this method. 

{\em The bias of SMG.}
To better understand the cause of this bias, consider the case ($M=2$) of \eqref{eq:lambd-stoch-prob} for simplicity. We can rewrite the problem for solving for convex combination weights as
\begin{equation}
    \arg\!\min\limits_{\lambda\in[0, 1]}\left\Vert \lambda\nabla f_1(x, \xi) +(1-\lambda)\nabla f_2(x, \xi) \right\Vert^2
\end{equation}
which admits the closed-form solution for $\lambda$  as
\begin{equation}
   \lambda^g(x, \xi)= \left[ \frac{\left(\nabla f_2(x, \xi) - \nabla f_1(x, \xi)\right)^\top \nabla f_2(x, \xi)}{\left\Vert \nabla f_1(x, \xi) - \nabla f_2(x, \xi) \right\Vert^2} \right]_{+, ^1_{\intercal}} \label{eq:stoch-lambda-bias}
\end{equation}
where $[x]_{+, ^1_{\intercal}} = \max(\min(x, 1), 0)$. It can be seen that the solution for $\lambda$ is non-linear in $\nabla f_1(x, \xi)$ and $\nabla f_2(x, \xi)$, which suggests that $\mathbb{E}[\lambda^g(x, \xi)]\neq \lambda^*(x)$ and thus $\mathbb{E}[g(x, \xi)]\neq d(x)$.

To ensure convergence, a recent approach proposed to replace the stochastic gradient $\nabla f_m(x, \xi)$ with its mini-batch version with the batch size growing with the number of iterations \citep{liu2021stochastic}. 
However, this may not be desirable in practice and often leads to sample inefficiency. In multi-objective reinforcement learning settings, this means running increasingly many number of roll-outs for policy gradient calculation, which may be infeasible. 


\section{Stochastic Multi-objective Gradient Descent With Correction}\label{sec:method}

In this section, we will first propose a new stochastic update that addresses the biased multi-gradient in MOO, extend it to the generic MOO setting, and then establish its convergence result.

\subsection{A basic algorithmic framework}\label{sec.basic-moco}

We start by discussing how to obtain $\nabla f_m(x)$ without incurring the bias issue. The key idea is to approximate true gradients of each objective using a `tracking variable', and use these approximations in finding optimal convex combination coefficients, similar to MGDA and SMG. 
At each iteration $k$, assuming we have access to $h_{k,m}$ which is a stochastic estimator of $\nabla f_m(x_k)$ (e.g., $h_{k,m}=\nabla f_m(x_k,\xi_k)$).
We obtain $\nabla f_m(x_k)$ by iteratively updating the `tracking' variable $y_{k,m} \in \mathbb{R}^{d}$ by
\begin{equation}\label{eq:y update}
    y_{k+1,m} = \Pi_{\mathcal{Y}_m}\Big( y_{k,m} - \beta_k \big( y_{k,m} -  h_{k,m} \big) \Big), ~ ~~m=1,2,\cdots,M, 
\end{equation}
where $k\in 0, 1, 2 \dots$ is the iteration index, $\beta_k$ is the step size and $\Pi_{\mathcal{Y}_m}$ denotes the projection to set $\mathcal{Y}_m = \{y\in\mathbb{R}^d~|~\|y\|\leq C_{y,m}\}$, for some constant $C_{y,m} > 0$. When the function $f_m$ is Lipschitz continuous with $L_m$, we can choose $C_{y,m} = L_m$ so that $C_{y,m} = \sup \|y_{k, m}\|=\sup \| \nabla f_m(x_k) \|$, for all $k$.

Under some assumptions on the stochastic gradients $h_{k, m}$ that will be specified in Section \ref{sec:convg-analysis}, we can show that for a given $x_k$, the recursion in \eqref{eq:y update} admits a unique fixed-point $y_m^*(x_k)$ that satisfies 
\begin{equation}
    y_m^*(x_k) = \E[ h_{k,m}]=\nabla f_m(x_k).
\end{equation}
In this subsection we will first assume that $h_{k,m}$ is an unbiased estimator of $\nabla f_m(x_k)$, and will generalize to the biased estimator in the next subsection. In this case, with only one sample needed at each iteration, the distance between $y_{m,k}$ and $\nabla f_m(x_k)$ is expected to diminish as the number of iteration increases. 

\begin{algorithm}[t]
	\caption{MoCo: Stochastic Multi-objective gradient with Correction }\label{alg:MoCo} 
	\begin{algorithmic}[1]
	    \State \textbf{Input} Initial model parameter $x_0$, tracking parameters $\{y_{0, i}\}_{m=1}^M$, convex combination coefficient parameter $\lambda_0$, and their respective learning rates $\{\alpha_k\}_{k=0}^{K}$,  and $\{\beta_k\}_{k=0}^{K}$, $\{\gamma_k\}_{k=0}^{K}$.
	    \For {$k=0, \dots, K-1$}
	        \For {objective $m=1, \dots, M$}
	            \State Obtain gradient estimator $h_{m,k}$ \Comment{either $h_{m,k}=\nabla f_m(x_k,\xi_k)$ or $h_{m,k}$ in \eqref{eq:z update-new}-\eqref{eq:z update2}}
	            \State Update $y_{k+1,m}$ following \eqref{eq:y update}
	        \EndFor
	    \State Update $\lambda_{k+1}$ and $x_{k+1}$ following  \eqref{eq:lambda update}-\eqref{eq:x update}
	    \EndFor
	    \State \textbf{Output} $x_K$
	\end{algorithmic} 
\end{algorithm}

Even with an accurate estimate of $\nabla f_m(x)$, solving \eqref{eq:prob-form} is still not easy since these gradients could conflict with each other. As described in Section \ref{sec:mgda}, given $x \in \mathbb{R}^d$, the MGDA algorithm finds the optimal scalars, denoted as $\{\lambda_m^*(x)\}_{m=1}^M$, to scale each gradient $\nabla f_m(x)$ such that $d(x) = \sum_{m=1}^M \lambda_m^*(x) \nabla f_m(x)$, and $-d(x)$ is a common descent direction for every $f_m(x)$. 
For obtaining the corresponding convex combinations when we do not have access to the true gradient, we propose to use $Y_k \coloneqq (y_{k,1},...,y_{k,M}) \in \mathbb{R}^{d \times M}$ as an approximation of $\nabla F(x_k)$. In general, the solution for \eqref{eq:lambd-det-prob} is not necessarily unique. Thus, we overcome this issue by adding $\ell_2$ regularization. Specifically, the new subproblem is given by
\begin{align}
    \lambda_\rho^*(x) = \arg\min_{\lambda} ~ \|\nabla F(x) \lambda\|^2 + \frac{\rho}{2}\|\lambda\|^2~~~{\rm s.~t.}~~~\lambda\in\Delta^M \coloneqq \{\lambda \in \mathbb{R}^M~|~ \mathbf{1}^\top \lambda = 1,~\lambda \geq 0\}, \label{eq:lambd-det-prob-rho}
\end{align}
where $\rho>0$ is a regularization constant. 
\begin{Remark}[On the Lipschitz continuity of $\lambda^*_\rho(x)$]\label{remark:lambda}
    Since the subproblems \eqref{eq:lambd-det-prob} and \eqref{eq:lambd-stoch-prob} depend on $x$, the subproblems change at each iteration. To analyse the convergence of the algorithm, it is important to quantify the change of solutions $\lambda^*(x), \lambda^g(x, \xi)$. One natural way of ensuring this is assuming the aforementioned solutions are Lipschitz continuous in $x$ (or equivalently in $\nabla F(x)$, if $~\nabla F(x)$ is Lipschitz in $x$); see \citep{liu2021stochastic}. However, this condition does not hold since  $\nabla F(x)$ is not positive definite and thus the solution to the subproblems \eqref{eq:lambd-det-prob} and \eqref{eq:lambd-stoch-prob} are not unique. We overcome this issue by adding the regularization   $\rho$ to ensure uniqueness of the solution of subproblem \eqref{eq:lambd-det-prob-rho} and the Lipschitz continuity of $\lambda_{\rho}^*(x)$ in $x$.
\end{Remark}

With this regularized reformulation, we find $\lambda^*_\rho(x)$ by the stochastic projected gradient descent of (\ref{eq:lambd-det-prob-rho}), given by
\begin{equation}\label{eq:lambda update}
    \lambda_{k+1} = \Pi_{\Delta^M}\left(\lambda_k - \gamma_k \left(Y_k^\top Y_k + \rho I \right) \lambda_k\right), 
\end{equation}
where $\gamma_k$ is the step size, $I\in\mathbb{R}^{M\times M}$ is the identity matrix, and  $\Pi_{\Delta^M}$ denotes the projection to the probability simplex $\Delta^M$. With $\lambda_k$ as an approximation of $\lambda^*_\rho(x_k)$ and $Y_k$ as an approximation of $\nabla F(x_k)$, we then update
\begin{equation}\label{eq:x update}
    x_{k+1} = x_k - \alpha_k Y_k \lambda_k,
\end{equation}
where $a_k$ is the stepsize. We have summarized the basic MoCo algorithm in Algorithm \ref{alg:MoCo}. 

\subsection{Generalization to nested MOO setting}\label{sec:moco-inexact}

In the previous section, we have introduced the gradient estimator $h_{k,m}$. In the simple case where $\nabla f_m(x,\xi)$ is obtained, setting $h_{k,m}=\nabla f_m(x_k,\xi_k)$ leads to the exact solution $y_m^*(x_k) = \nabla f_m(x_k)$, thus solving the bias problem. 
However, in some practical applications, $\nabla f_m(x,\xi)$ may be also difficult to obtain, hence $h_{k,m}$ can be biased. In this section, we will establish the tolerance for such bias. 

To put this on concrete ground, we first 
consider the following nested multi-objective problem:
\begin{align}\label{eq:nested problem}
    &\min\limits_{x\in \mathbb{R}^d} F(x) \coloneqq \left(\E_\xi[f_1(x,z_1^*(x), \xi)], \E_\xi[f_2(x, z_2^*(x),\xi)], \dots, \E_\xi[f_M(x,z_M^*(x), \xi)] \right) \nonumber\\
    &\quad {\rm s.t. }\quad z_m^*(x) \coloneqq \arg\min_{z \in \mathbb{R}^d} l_m(x,z) \coloneqq \E_\varphi[l_m(x,z,\varphi)],~~~m=1,2,\cdots,M
\end{align}
where $l_m$ is a strongly-convex function, and $\varphi$ is a random variable.
For convenience, we define $f_m(x,z) \coloneqq \E_\xi[f_m(x,z, \xi)]$ and $f_m(x)\coloneqq f_m(x,z_m^*(x))$.
Under some conditions that will be specified later, it has been shown in \citep{ghadimi2018approximation} that the gradient of $f_m(x)$ takes the following form:
\begin{align}\label{eq:nabla f_m(x)}
    \nabla f_m(x) = \nabla_x f_m(x,z_m^*(x)) 
    - \nabla_{xz}^2 l_m(x,z_m^*(x)) [\nabla_{zz}^2 l_m(x,z_m^*(x))]^{-1} \nabla_z f_m(x,z_m^*(x))
\end{align}
where $\nabla_x f(x,z_m^*(x)) = \frac{\partial f(x,z)}{\partial x}|_{z=z_m^*(x)}$, $\nabla_{xz}^2 l(x,z_m^*(x)) = \frac{\partial l(x,z)}{\partial x \partial z}|_{z=z_m^*(x)}$ and likewise for $\nabla_z f(x,z_m^*(x))$ and $\nabla_{zz}^2 l(x,z_m^*(x))$.
Computing the unbiased stochastic estimate of \eqref{eq:nabla f_m(x)} requires $z_m^*(x)$, which is often costly in practice. Instead, we can estimate $z_m^*(x_k)$ via a nested loop update at each iteration $k$, given by 
\begin{align}\label{eq:z update-new}
    z_{t+1,k, m} = z_{t, k,m} - \eta_t \nabla_z l_m(x_k,z_{t, k,m},\varphi_{t, k}),~~~t=1, 2, \cdots, T~~~{\rm with}~~~z_{1, k, m}:=z_{k-1, m}
\end{align}
where $T\in \mathbb{N}$ is the total number of nested loop steps and $\eta_t$ is the stepsize. The estimate for $z_m^*(x_k)$ is then obtained as $z_{k, m} := z_{T, k, m}$, which we use to replace $z_m^*(x_k)$ in \eqref{eq:nabla f_m(x)} to compute a biased gradient estimator as 
\begin{align}\label{eq:z update2}
    h_{k,m} = \nabla_x f_m(x_k,z_{k,m},\xi_k) -\nabla_{xz}^2 l_m(x_k,z_{k,m},\varphi_k') H^{zz}_{k,m} \nabla_z f_m(x_k,z_{k,m},\xi_k) 
\end{align}
where $\varphi_{t, k}$, $\varphi_k'$ have the same distribution as that of $\varphi$, and $H_{k,m}^{zz}$ is a stochastic approximation of the Hessian inverse $[\nabla_{zz}l_m(x_k,z_{k,m})]^{-1}$. Given $x_k$, when $z_{k,m}$ reaches the optimal solution $z_m^*(x_k)$, it follows from \eqref{eq:nabla f_m(x)} that $\E[h_{k,m}] = \nabla f_m(x_k)$. We summarize the algorithm for MoCo with nested loop explicitly in Algorithm \ref{alg:MoCo-inexact} (see Appendix \ref{app:moco-inexact}). Next, we quantify the error when $z_{k,m}$ is non-optimal. 

To this end, we first consider the single-loop case where the nested loop update consists of a single step, i.e. $T=1$, which has been considered in the previous bilevel optimization literature \citep{hong2020two, chen2021single}. Following lemma provides the error in the stochastic gradient estimator in \eqref{eq:z update2} for this case. 
\begin{Lemma}\label{lemma:hkm error}
Define $\mathcal{F}_k$ as the $\sigma$-algebra generated by $Y_1,Y_2,...,Y_k$. Consider the sequences generated by \eqref{eq:y update}, \eqref{eq:lambda update}, \eqref{eq:x update}, \eqref{eq:z update-new} and \eqref{eq:z update2}. If we choose $T=1$ and $\eta_1=\beta_K$, under assumptions specified in Appendix \ref{app:moco-inexact}, we have for any $m$ that
\begin{subequations}
\begin{align}
\!\!    \frac{1}{K}\sum_{k=1}^K \E\left[\|\E[h_{k,m}|\mathcal{F}_k]-\nabla f_m(x_k)\|^2\right] &= \mathcal{O}\Big(\frac{\alpha_K^2}{\beta_K^2}\Big),~~ ~ \text{\rm and} \label{eq:lem1-1}\\
\E\left[\|h_{k,m}-\E\left[h_{k,m}|\mathcal{F}_k\right]\|^2|\mathcal{F}_k\right] &\leq \sigma_0^2, \label{eq:lem1-2}
\end{align}
\end{subequations}
where $\alpha_k, \beta_k$ are the learning rates in updates \eqref{eq:x update}, \eqref{eq:y update} respectively, and $\sigma_0>0$ is a constant. 
\end{Lemma}
Lemma \ref{lemma:hkm error} shows that the average bias of the gradient estimator will diminish if $\alpha_k$ and $\beta_k$ are chosen properly. In addition, the variance of the estimator is also bounded by a constant.  Note that the lower-level update \eqref{eq:z update-new} is performed simultaneously with the upper level update \eqref{eq:x update}. This update allows reduced number of samples used at each iteration at the expense of larger bias in the  stochastic gradient estimator. Instead, it is also possible to update lower-level parameter $z_{k, m}$ using multiple steps ($T>1$) so that the bias of the stochastic gradient will be small; see also \citep{ghadimi2018approximation, yang2021provably, chen2021tighter}. The bias of stochastic gradient estimator in  this case will be established next. 
\begin{Lemma}\label{lemma:hkm error-new}
Define $\mathcal{F}_k$ as the $\sigma$-algebra generated by $Y_1,Y_2,...,Y_k$. Consider the sequences generated by \eqref{eq:y update}, \eqref{eq:lambda update}, \eqref{eq:x update}, \eqref{eq:z update-new} and \eqref{eq:z update2}. If we choose $T=\mathcal{O}\left(\frac{1}{\beta_K}\right)$, with a suitable choice of $\eta_t$ and under some assumptions specified in Appendix \ref{app:moco-inexact}, we have for any $m$ that
\begin{subequations}
\begin{align}
  \frac{1}{K}\sum_{k=1}^K \E\left[\|\E[h_{k,m}|\mathcal{F}_k]-\nabla f_m(x_k)\|^2\right] &= \mathcal{O}(\beta_K),~~ ~ \text{\rm and} \label{eq:lem2-1}\\
\E\left[\|h_{k,m}-\E\left[h_{k,m}|\mathcal{F}_k\right]\|^2|\mathcal{F}_k\right] &\leq \sigma_0^2, \label{eq:lem2-2}
\end{align}
\end{subequations}
where $\beta_k$ is the learning rate in update \eqref{eq:y update}, and $\sigma_0>0$ is a constant. 
\end{Lemma}
Lemma \ref{lemma:hkm error-new} shows that the average bias of the gradient estimator will diminish if $\beta_k$ is chosen properly. In addition, the variance of the estimator is also bounded by a constant, similar to the simultaneous update.  Allowing biased gradient facilitates MoCo to tackle the more challenging MTL task as highlighted below.

\begin{Remark}[Connection between nested MOO with multi-objective actor-critic]
Choosing each $f_m$ in \eqref{eq:nested problem} to be the infinite-horizon accumulated reward and each $l_m$ to be the critic objective function will lead to the popular actor-critic algorithm in reinforcement learning \citep{konda1999actor, wen2021characterizing}. In this work, we have extended this to the multi-objective case, and conducted experiments on multi-objective soft actor critic in Appendix \ref{app:rl}.
\end{Remark}
\vspace{-0.1cm}
\subsection{A unified convergence result} \label{sec:convg-analysis}

In this section we provide the convergence analysis for our proposed method. First, we make following assumptions on the objective functions.

\begin{Assumption}\label{assumption:lip}
For $m\in[M]$: $f_m(x)$ is Lipschitz continuous with modulus $L_m$ and $\nabla f_m(x)$ is Lipschitz continuous with modulus $L_{m,1}$ {for any $x\in\mathbb{R}^d$.}
\vspace{-0.2cm}
\end{Assumption}
Due to the $x$ update in \eqref{eq:x update}, the optimal solution for $y$ and $\lambda$ sequences are changing at each iteration, and the change roughly scales with $\|x_{k+1}-x_k\|$. In order to guarantee the convergence of $y$ and $\lambda$, the change in optimal solution needs to be controlled. The first half of Assumption \ref{assumption:lip} ensures that $\nabla f(x)$ is uniformly bounded, that is, $\|x_{k+1}-x_k\|$ is upper bounded and thus controlled. The second half of the assumption is standard in  establishing the convergence of non-convex functions \citep{bottou2018optimization}.
Next, to unify the analysis of the nested MOO in Section \ref{sec:moco-inexact} and the basic MOO in Section \ref{sec.basic-moco}, we make the following assumptions on the quality of the gradient estimator $h_{k,m}$. These assumptions unify the nested MOO setting with the basic single level MOO setting with biased stochastic gradients.
\begin{Assumption}\label{assumption:h}
For any $m$, there exist constants $c_m,\sigma_m$ such that $\frac{1}{K}\sum_{k=1}^K\E\|\E[ h_{k,m}|\mathcal{F}_k]\!-\!\nabla f_m(x_k)\|^2\leq c_m \alpha_K^2/\beta_K^2$ and $\E[\|h_{k,m}\!-\!\E[ h_{k,m}|\mathcal{F}_k]\|^2 |\mathcal{F}_k]\!\leq\! \sigma_m^2$ for any $k$. 
\end{Assumption}


\begin{Assumption}\label{assumption:h-new}
For any $m$, there exist constants $c_m,\sigma_m$ such that $\frac{1}{K}\sum_{k=1}^K\E\|\E[ h_{k,m}|\mathcal{F}_k]\!-\!\nabla f_m(x_k)\|^2\leq c'_m \beta_K$ and $\E[\|h_{k,m}\!-\!\E[ h_{k,m}|\mathcal{F}_k]\|^2 |\mathcal{F}_k]\!\leq\! \sigma_m^2$ for any $k$. 
\end{Assumption}


Assumptions \ref{assumption:h} and \ref{assumption:h-new} require the stochastic gradient $h_{k,m}$ almost unbiased and has bounded variance, which are weaker assumptions compared to the standard unbiased stochastic gradient assumption, i.e., \citep[Assumption 5.2]{liu2021stochastic},  because i) they do not require the variance $\sigma_m^2$ to decrease in the same speed as $\alpha_k^2$; and ii) they allow bias in the stochastic gradient of each objective function.
In practice, the batch size is often fixed, and thus the variance is non-decreasing, which suggests one benefit of Assumptions \ref{assumption:h} and Assumptions \ref{assumption:h-new} over that in \citep{liu2021stochastic}. 
In nested MOO setting in Section \ref{sec:moco-inexact}, Assumption \ref{assumption:h} can be satisfied by running simultaneous upper and lower level updates (see Lemma \ref{lemma:hkm error}), and Assumption \ref{assumption:h-new} can be satisfied by running multiple nested updates (see Lemma \ref{lemma:hkm error-new}), which require  additional lower-level   samples.


Next we provide convergence results for MoCo under Assumptions \ref{assumption:lip} and \ref{assumption:h}. Thus, the results provided in this section hold for nested MOO with simultaneous update (Lemma \ref{lemma:hkm error}). To this end, we first provide the following Lemma which shows that, for some choice of step size and regularization parameter, the stochastic multi-gradient estimate used in update \eqref{eq:x update} in MoCo converge to the true MGDA direction given in \eqref{eq:mgda}.

\begin{Lemma}\label{lemma:lower}
Consider the sequences generated by Algorithm \ref{alg:MoCo}. 
Under Assumptions \ref{assumption:lip} and \ref{assumption:h}, if we choose step sizes $\alpha_K = \Theta(K^{-\frac{3}{4}})$, $\beta_K=\Theta(K^{-\frac{1}{2}})$, $\gamma_K=\Theta(K^{-\frac{1}{3}})$,  and $\rho = \Theta(K^{-\frac{1}{6}})$,  it holds that
\begin{equation}
    \frac{1}{K}\sum_{k=1}^K\E\left[\|d(x_k) - Y_k \lambda_k\|^2\right] = \mathcal{O}(M K^{-\frac{1}{6}}).
\end{equation}
\end{Lemma}
With suitable choice of $\rho$, $\lambda_k$ and $Y_k$ converge to {$\lambda^*_\rho(x_k)$} and $\nabla F(x_k)$ respectively. As a result, the update direction $Y_k \lambda_k$ for $x_k$ converges to $d(x_k)=\nabla F(x_k) \lambda^*(x_k)$, which is the desired MGDA direction in the mean square sense.  Next we establish the convergence of MoCo to a Pareto stationary point following the aforementioned descent direction.

\begin{Theorem}\label{theorem:upper}
Consider the sequences generated by Algorithm \ref{alg:MoCo}. Under Assumptions \ref{assumption:lip} and \ref{assumption:h}, if we choose step sizes $\alpha_K = \Theta(K^{-\frac{3}{4}})$, $\beta_K=\Theta(K^{-\frac{1}{2}})$, $\gamma_K=\Theta(K^{-\frac{1}{3}})$,  and $\rho = \Theta(K^{-\frac{1}{6}})$, it holds that
\begin{equation}
    \frac{1}{K}\sum_{k=1}^K \E\left[\|\nabla F(x_k)\lambda^*(x_k)\|^2\right] =  \mathcal{O}(M^{\frac{1}{2}}K^{-\frac{1}{12}}).
\end{equation}
\end{Theorem}
Theorem \ref{theorem:upper} shows that the MGDA direction $\nabla F(x_k)\lambda^*(x_k)$ converges to $0$, which indicates that the proposed MoCo method is able to achieve Pareto-stationarity. To the best of our knowledge, this is the first finite-time convergence guarantee for the stochastic MGDA method under non-convex objective functions. 

\begin{Remark}[Comparison with SMG \citep{liu2021stochastic}]
Theorem \ref{theorem:upper} provides the convergence rates of the MoCo algorithm under Assumptions \ref{assumption:lip} and \ref{assumption:h}. 
Compared to the convergence analysis of SMG in \citep{liu2021stochastic}, the convergence rates in Theorem \ref{theorem:upper} is derived under more natural and practical assumptions which additionally account for the  non-convex nested MOO setting. One such assumption is $\lambda^*(x)$ being Lipschitz in $x$. This may not be true unless $\nabla F(x)$ is full rank which does not hold at Pareto stationary points. 
We overcome this problem by adding regularization and carefully choosing the regularization parameter. 
\end{Remark}

\subsection{Convergence with stronger assumptions}
In this section, we   provide convergence results  to a Pareto stationary point that improve upon Theorem \ref{theorem:upper} under some additional assumptions on the bounded function value of $F(x_k)$.

\begin{Theorem}\label{theorem:upper-fbound}
Consider the sequences generated by Algorithm \ref{alg:MoCo}. Furthermore assume there exists a constant $F>0$ such that for all $k\in[K]$,  $\|F(x_k)\|\leq F$. Then, under Assumptions \ref{assumption:lip} and \ref{assumption:h}, if we choose step sizes {$\alpha_k = \Theta(K^{-\frac{3}{5}})$, $\beta_k=\Theta(K^{-\frac{2}{5}})$, $\gamma_k=\Theta(K^{-1})$, and $\rho = 0$}, it holds that
\begin{equation}
    \frac{1}{K}\sum_{k=1}^K \E\left[\|\nabla F(x_k)\lambda^*(x_k)\|^2\right] =  {\mathcal{O}\Big(MK^{-\frac{2}{5}}\Big)}.
\end{equation}
\end{Theorem}

Theorem \ref{theorem:upper-fbound} shows Algorithm \ref{alg:MoCo} will converge to a Pareto stationary point with an improved convergence rate, if the sequence of functions $F(x_1), F(x_2), \dots, F(x_k)$ are bounded. Note that since this result holds for Assumption \ref{assumption:h}, this result also covers the nested MOO with simultaneous update (Lemma \ref{lemma:hkm error}). Next, we provide an improved convergence rate for MoCo under Assumption \ref{assumption:h-new}.

\begin{Theorem}\label{theorem:upper-fbound-stronger}
Consider the sequences generated by Algorithm \ref{alg:MoCo}. Furthermore assume there exists a constant $F>0$ such that for all $k\in[K]$,  $\|F(x_k)\|\leq F$. Then, under Assumptions \ref{assumption:lip} and \ref{assumption:h-new}, if we choose step sizes {$\alpha_k = \Theta(K^{-\frac{1}{2}})$, $\beta_k=\Theta(K^{-\frac{1}{2}})$, $\gamma_k=\Theta(K^{-\frac{3}{4}})$, and $\rho = 0$}, it holds that
\begin{equation}
    \frac{1}{K}\sum_{k=1}^K \E\left[\|\nabla F(x_k)\lambda^*(x_k)\|^2\right] =  {\mathcal{O}\Big(MK^{-\frac{1}{2}}\Big)}.
\end{equation}    
\end{Theorem}

Theorem \ref{theorem:upper-fbound-stronger} improves upon the convergence rate presented in  Theorem \ref{theorem:upper-fbound} using Assumption \ref{assumption:h-new}. 

\begin{Remark}[On the stronger assumptions]
Both Theorems \ref{theorem:upper-fbound} and \ref{theorem:upper-fbound-stronger} provide the convergence rates of the MoCo algorithm under the stronger condition that the sequence of objective values at $x_k$ are bounded. This assumption is also used in other stochastic analysis works such as \citep{levy2021storm+}. Furthermore, this kind of condition is naturally satisfied in practical applications of stochastic non-convex optimization such as reinforecement learning \citep{wu2020finite}. In order for the stochastic gradient estimator to satisfy the Assumption \ref{assumption:h-new} used in Theorem \ref{theorem:upper-fbound-stronger} in the nested MOO setting, one need to run multiple nested loop steps to achieve the required bias. This kind of multi-step inner loop sampling is used in Actor-critic methods \citep{konda1999actor}, where multiple update steps are used for critic update to increase the stability in critic.
\end{Remark}

\section{Related work}\label{sec:related-work}

 To put our work in context, we review prior art that we group in the following two categories.

\textbf{Multi task learning.} MTL algorithms find a common model that can solve multiple possibly related tasks. MTL has shown great success in many fields such as natural language processing, computer vision and robotics \citep{hashimoto2016joint}, \citep{ruder2017overview}, \citep{zhang2021survey}, \citep{vandenhende2021multi}. One line of research involves designing machine learning models that facilitate MTL, such as architectures with task specific modules \citep{misra2016cross}, with attention based mechanisms \citep{rosenbaum2017routing}, \citep{yang2020multi}, or with different path activation corresponding to different tasks. Our method is model agnostic, and thus can be applied to these methods in a complementary manner. Another line of work focuses on decomposing a problem into multiple local tasks and learn these tasks using smaller models \citep{rusu2015policy}, \citep{parisotto2015actor}, \citep{teh2017distral}, \citep{ghosh2017divide}. These models are then aggregated into a single model using knowledge distillation \citep{hinton2015distilling}. Our method does not require multiple models in learning, and focus on learning different tasks simultaneously using a single model.  

Furthermore, recent works in MTL have analyzed MTL from different viewpoints. In  \citep{wang2021bridging}, the authors explore the connection between gradient based meta learning and MTL. In \citep{ye2021multi}, the meta learning problem with multiple objectives in the upper levels has been tackled via a gradient based MOO approach. The importance of task grouping in MTL is analysed in works such as \citep{fifty2021efficiently}. In \citep{meyerson2020traveling}, the authors show that seemingly unrelated tasks can be used for MTL. Our proposed method is orthogonal to these approaches, and can be combined with theses methods in stochastic settings to achieve better performance.  

\textbf{Gradient-based MOO.} This line of work involves optimizing multiple objectives simultaneously using gradient manipulations. A foundational algorithm in this regard is MGDA\citep{Desideri2012mgda}, which dynamically combine gradients to find a common descent direction for all objectives. A comprehensive convergence analysis for the deterministic MGDA algorithm has been provided in \citep{fliege2019complexity}. Recently, \citep{liu2021stochastic} extends this analysis to the stochastic counterpart of multi-gradient descent algorithm, for smooth convex and strongly convex functions. However, this work makes strong assumptions on the bias of the stochastic gradient and does not consider the nested MOO setting that is central to the multi-task reinforcement learning. In \cite{yang2021pareto}, the authors establish convergence of stochastic MGDA under the assumption of access to true convex combination coefficients, which may not be true in a practical stochastic optimization setting. In \citep{gu2022min}, the authors propose a bi-level multi-objective min-max optimization algorithm for minimizing worst objective at each iteration, instead of considering the worst descent direction as in MGDA. Another related line of work considers the optimization challenges related to MTL, considering task losses as objectives. One common approach is to find gradients for balancing learning of different tasks. The simplest way is to re-weight per task losses based on a specific criteria such as uncertainty \citep{kendall2017multi}, gradient norms \citep{chen2018gradnorm} or task difficulty \citep{guo2018dynamic}. These methods are often heuristics and may be unstable. More recent work \citep{sener2018multi}, \citep{yu2020gradient}, \citep{liu2021conflict}, \citep{gu2021adversarial} introduce gradient aggregation methods which mitigate conflict among tasks while preserving utility. In \citep{sener2018multi}, MTL has been first tackled through the lens of MOO techniques using MGDA. In \citep{yu2020gradient}, a new method called PCGrad has been developed to amend gradient magnitude and direction in order to avoid  conflicts among per task gradients. In \citep{liu2021conflict}, an algorithm  similar to MGDA, named CAGrad, has been developed,  which uniquely minimizes the average task loss. In \citep{liu2021towards}, an impartial objective gradient modification mechanism has been studied. A Nash bargaining solution for MTL has been proposed in \citep{navon2022multi} for weighting per objective gradients.  

All the aforementioned works on MTL use the deterministic objective gradient for analysis (if any), albeit the accompanying empirical evaluations are done in a stochastic setting. There are also gradient-based MOO algorithms that find a set of Pareto optimal points for a given problem rather than one. To this end, works such as \citep{liu2021profiling, liu2021stochastic, lin2019pareto, mahapatra2021exact, navon2020learning} develop algorithms that find multiple points in the Pareto front, ensuring some quality of the obtained Pareto points. Our work is orthogonal to this line of research, and can potentially be combined with those ideas to achieve better performance.

\textbf{Non gradient-based MOO.} In addition to gradient based MOO which is the main focus of this paper, there also exist non gradient based blackbox MOO algorithms such as \citep{deb2002fast, golovin2020random, knowles2006parego, konakovic2020diversity}, which are based on an evolutionary algorithm or Bayesian optimization. However, these methods often suffer from the curse of dimensionality, and may not be feasible in large scale MOO problems. 

  

\begin{figure}[t]
     \centering
     \begin{subfigure}[b]{0.19\textwidth}
         \centering
         \includegraphics[width=\textwidth]{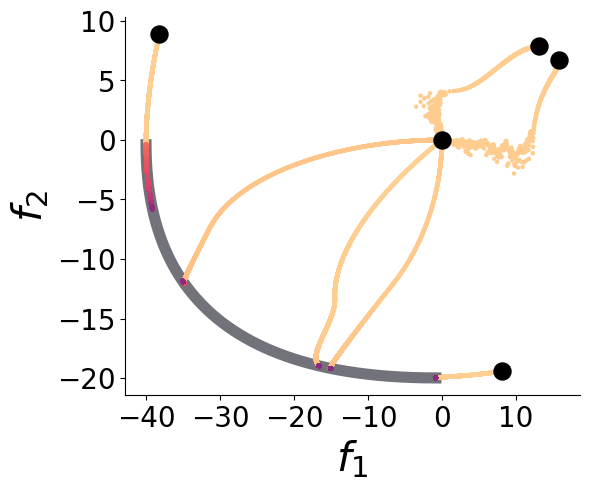}
         \caption{MGDA}
         \label{fig:front-mgda}
     \end{subfigure}
     \hfill
     \begin{subfigure}[b]{0.17\textwidth}
         \centering
         \includegraphics[width=\textwidth]{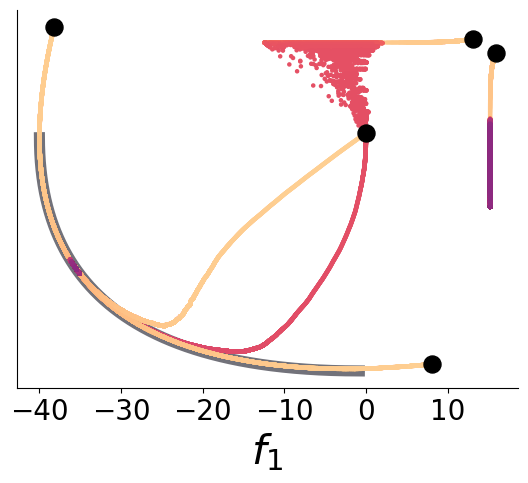}
         \caption{SMG}
         \label{fig:front-smgd}
     \end{subfigure}
     \hfill
     \begin{subfigure}[b]{0.17\textwidth}
         \centering
         \includegraphics[width=\textwidth]{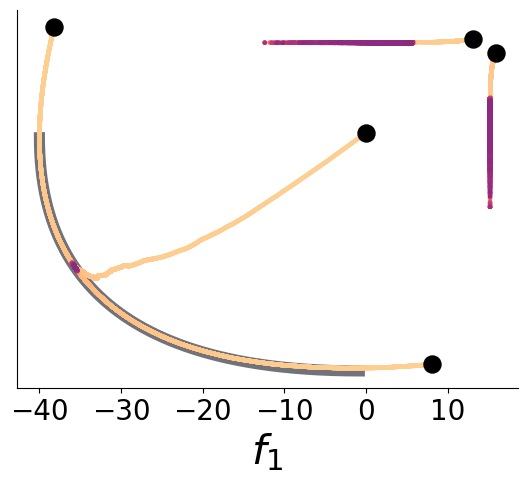}
         \caption{PCGrad}
         \label{fig:front-pcgrad}
     \end{subfigure}
     \hfill
     \begin{subfigure}[b]{0.17\textwidth}
         \centering
         \includegraphics[width=\textwidth]{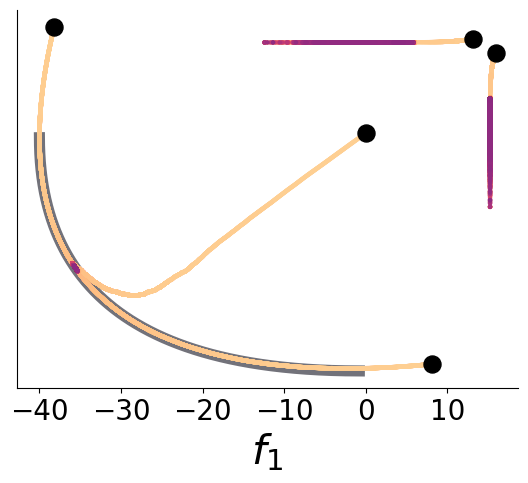}
         \caption{CAGrad}
         \label{fig:front-cagrad}
     \end{subfigure}
     \begin{subfigure}[b]{0.19\textwidth}
         \centering
         \includegraphics[width=\textwidth]{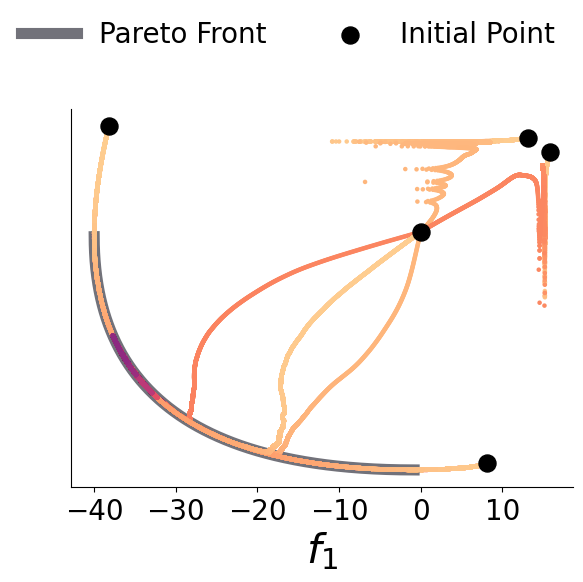}
         \caption{MoCo (ours)}
         \label{fig:front-tracking}
     \end{subfigure}
        \caption{Comparison of trajectories in the objective space. We use five initializations in the same toy example in Figure \ref{fig:toy-comp}, and plot the optimization trajectory in the objective space. MGDA converges to the Pareto front from all of the initializations. SMG, PCGrad, and CAGrad which only have access to single stochastic gradient per objective fail to converge to the Pareto front in some initializations. Our MoCo follows a similar trajectory to that of MGDA, and finds the Pareto front for each initialization.}
        \label{fig:front-comp}
\end{figure}

 \begin{table}[tb]
\small
    \centering
    \setlength{\tabcolsep}{1.0em} 
{\renewcommand{\arraystretch}{1.3}
    \begin{tabular}{c c c c c c}
    \hline
         \multirow{3}{*}{Method} & \multicolumn{2}{c}{Segmentation} & \multicolumn{2}{c}{Depth} & \multirow{3}{*}{$\Delta m\% \downarrow$}  \\ \cline{2-3}\cline{4-5}
         & \multicolumn{2}{c}{(Higher Better)} & \multicolumn{2}{c}{(Lower Better)} &\\
         & mIoU & Pix Acc & Abs Err & Rel Err & \\ \hline\hline
         Independent &            74.01       & 93.16             & 0.0125 & 27.77 & -  \\ \hline
         Cross-Stitch \citep{misra2016cross} & 73.08        & 92.79          & 0.0165 & 118.5 &  90.02\\
         MTAN \citep{liu2019end} &                75.18        & 93.49          & 0.0155 & 46.77 &  22.60\\
         MGDA \citep{sener2018multi}  &                68.84       & 91.54          & 0.0309 & \textbf{33.50} &  44.14\\
         PCGrad \citep{yu2020gradient} &              75.13       & 93.48          & 0.0154 & 42.07 &  18.29\\
         GradDrop \citep{chen2020just} &            75.27       & 93.53    & 0.0157 & 47.54 &   23.73\\
         CAGrad \citep{liu2021conflict} &             75.16       & 93.48          & \textbf{0.0141} & 37.60 &     11.64\\
         \textbf{MoCo (ours)} & \textbf{75.42} & \textbf{93.55} & 0.0149 & 34.19 &  \textbf{9.90}\\ \hline
    \end{tabular}}
    \vspace{5 pt}
    \caption{Multi-task supervised learning on CityScape dataset with the 7-class semantic segmentation and depth estimation results. Results are averaged over 3 independent runs. CAGrad, PCGrad, GradDrop and our method are applied on the MTAN backbone.}
    \label{tab:cityscape-results}
\end{table}

\section{Experiments} \label{sec:experiments}

In this section,   we first provide further illustration of  MoCo in comparison with existing gradient-based MOO algorithms in the toy example. Then we provide empirical comparison of our proposed method with the state-of-the-art MTL algorithms, using challenging and widely used real world MTL benchmarks in supervised and reinforcement learning settings. The  details of hyperparameters are provided in Appendix \ref{app:experiments}.
 
\textbf{Toy example.}
To further elaborate on how MoCo converges to a Pareto stationary point, we again optimize the two objectives given in Figure \ref{fig:toy-comp} and demonstrate the performance in the objective space (Figure \ref{fig:front-comp}). MGDA with true gradients converges to a Pareto stationary point in all initializations. However, it can be seen that SMG, PCGrad,and CAGrad methods fail to converge to a Pareto stationary point, and end up in dominated points in the objective space for some initializations. This is because these algorithms use a biased multi-gradient that does not become zero. In contrast, MoCo converges to Pareto stationary points in every initialization, and follows a similar trajectory to MGDA. 

 \vspace{-0.2cm}
\subsection{Supervised learning}
We compare MoCo with existing MTL algorithms using NYU-v2 \citep{silberman2012indoor} and CityScapes \citep{cordts2015cityscapes} datasets. 
We follow the experiment setup of \citep{liu2021conflict} and combine our method with MTL method MTAN \citep{liu2019end}, which applies an attention mechanism. We evaluate our method in comparison to CAGrad, PCGrad, vanilla MTAN and Cross-Stitch \citep{misra2016cross}. Following \citep{maninis2019attentive, liu2021conflict, navon2022multi}, we use the per-task performance drop of a metric $S_m$ for method $\mathcal{A}$ with respect to baseline $\mathcal{B}$ as a measure of the overall performance of a given method. Formally, it is given by $\Delta m = \frac{1}{M}\sum_{m=1}^{M}(-1)^{\ell_m}\left(S_{\mathcal{A},m} - S_{\mathcal{B}, m}\right)/S_{\mathcal{B}, m}$ , where $M$ is the number of tasks,  $S_{\mathcal{B}, m}$ and $S_{\mathcal{A}, m}$ are the values of metric $S_m$ obtained by the baseline and the compared method respectively. Here, $\ell_m=1$ if higher values for $S_m$ are better and $0$ otherwise.

The results of the experiments are shown in Table \ref{tab:cityscape-results} and \ref{tab:nyu-v2-results}. Our method, MoCo, outperforms all the existing MTL algorithms in terms of $\Delta m \%$ for both Cityscapes and NYU-v2 datasets. Since our method focuses on the gradient correction, our method can also be applied on top of existing gradient based MOO methods. Additional experiment results on this regard are provided in Appendix \ref{app:experiments}.

\begin{table}[tb]
\vspace{-0.2cm}
\small
    \centering
    \setlength{\tabcolsep}{0.3em} 
{\renewcommand{\arraystretch}{1.4}
    \begin{tabular}{c c c c c c c c c c c}
    \hline
         \multirow{3}{*}{Method} & \multicolumn{2}{c}{Segmetation} & \multicolumn{2}{c}{Depth} & \multicolumn{5}{c}{Surface Normal} & \multirow{3}{*}{$\Delta m \% \downarrow$} \\ \cline{2-3}\cline{4-5}\cline{6-10}
         & \multicolumn{2}{c}{(Higher Better)} & \multicolumn{2}{c}{(Lower Better)} & \multicolumn{2}{c}{\makecell{Angle Distance \\ (Lower Better)}} & \multicolumn{3}{c}{\makecell{Within $t^\circ$ \\ (Higher better)}}\\
         & mIoU & Pix Acc & Abs Err & Rel Err & Mean & Median & 11.25 & 22.5 & 30\\ \hline\hline
         Independent            & 38.30 & 63.76 & 0.6754 & 0.2780 & 25.01 & 19.21 & 30.14 & 57.20 & 69.15 & - \\ \hline
         Cross-Stitch \citep{misra2016cross} & 37.42 & 63.51 & 0.5487 & \textbf{0.2188} & 28.85 & 24.52 & 22.75 & 46.58 & 59.56 & 6.96 \\
         MTAN   \citep{liu2019end}    & 39.29 & 65.33 & 0.5493 & 0.2263 & 28.15 & 23.96 & 22.09 & 47.50 & 61.08 & 5.59 \\
         MGDA \citep{sener2018multi}               & 30.47 & 59.90 & 0.6070 & 0.2555 & \textbf{24.88} & \textbf{19.45} & \textbf{29.18} & \textbf{56.88} & \textbf{69.36} & 1.38 \\
         PCGrad \citep{yu2020gradient}           & 38.06 & 64.64 & 0.5550 & 0.2325 & 27.41 & 22.80 & 23.86 & 49.83 & 63.14 & 3.97 \\
         GradDrop \citep{chen2020just}       & 39.39 & 65.12 & \textbf{0.5455} & 0.2279 & 27.48 & 22.96 & 23.38 & 49.44 & 62.87 & 3.58 \\
         CAGrad \citep{liu2021conflict}     & 39.79 & 65.49 & 0.5486 & 0.2250 & 26.31 & 21.58 & 25.61 & 52.36 & 65.58 & 0.20\\
         \textbf{MoCo (ours)}     & \textbf{40.30}  & \textbf{66.07}  & 0.5575  & 0.2135 & 26.67  & 21.83  & 25.61  & 51.78 & 64.85 &  \textbf{0.16}\\ \hline
    \end{tabular}}
    \vspace{5 pt}
    \caption{Multi-task supervised learning on NYU-v2 dataset with 13-class semantic segmentation, depth estimation, and surface normal prediction results on NYU-v2 dataset. Results are averaged over 3 independent runs. CAGrad, PCGrad, GradDrop and our method are applied on the MTAN backbone.}
    \label{tab:nyu-v2-results}
\end{table}

\subsection{Reinforcement learning}\label{app:rl}
For the multi task reinforcement learning setting, we use the multi task reinforcement learning benchmark MT10 available in Met-world environment \citep{yu2020meta}. 
We follow the experiment setup used in \citep{liu2021conflict} and provide the empirical comparison between our MoCo method and the existing baselines.  Specifically, we use MTRL  codebase \citep{Sodhani2021MTRL} and use soft actor critic (SAC) \citep{haarnoja2018soft} as the underlying reinforcement learning algorithm. All the methods are trained for 2 million steps with a batch size of 1280. Each method is evaluated once every 10000 steps and the highest average test performance of a method over 5 random seeds over the entire training stage is reported in Table \ref{tab:mt10-results}.  
In this experiment, the vanilla MoCo outperforms PCGrad, but its performance is not as good as CAGrad that optimizes the average performance of all tasks.
We further run the gradient correction of MoCo on top of CAGrad, and the resultant algorithm outperforms the vanilla CAGrad. This suggests that incorporating the gradient correction of MoCo in existing gradient based MTL algorithms also boosts their performance. 

\begin{table}[htb]
\small
    \centering
    \setlength{\tabcolsep}{1.0em} 
{\renewcommand{\arraystretch}{1.3}
    \begin{tabular}{c c}
    \hline
         Method  & success
          (mean $\pm$ stderr)\\
         \hline\hline
         Multi-task SAC & $0.49 \pm 0.073$\\
         Multi-task SAC + Task Encoder& $0.54 \pm 0.047$ \\
         Multi-headed SAC & $0.61 \pm 0.036$\\
         PCGrad & $0.72 \pm 0.022$\\
         CAGrad & $0.83 \pm 0.045$\\ \hline
         \textbf{MoCo (ours)} & $0.75 \pm 0.050$\\ 
         \textbf{CAGrad + MoCo (ours)} & $0.86 \pm 0.022$\\ \hline
         One SAC agent per task (upper bound) & $0.90 \pm 0.032$\\ \hline
    \end{tabular}}
    \vspace{5 pt}
    \caption{Multi-task reinforcement learning results on MT10 Metworld benchmark.}
    \label{tab:mt10-results}
            \vspace{-0.4cm}
\end{table}


\section{Conclusions} \label{sec:conclusion}
We introduced a new algorithm termed MoCo for multi-objective gradient descent in the stochastic setting. This algorithm serves as an approach to addressing the known bias issue in the stochastic multi-gradient. We theoretically analysed and proved convergence of the algorithm to a Pareto stationary point, for smooth non-convex objectives. We also provided a generalization of this algorithm to the nested MOO setting. We then conducted empirical evaluation of our method and showed that MoCo can outperforms state-of-the-art MTL algorithms in challenging MTL benchmarks. One limitation of this algorithm is we cannot control the point to which MoCo will converge to. We showed that it is possible to combine MoCo with existing controlled gradient based MTL algorithms, which will further improve the performance.

\bibliography{bib}
\bibliographystyle{plainnat}



\appendix

\section{Proofs for Nested MOO}\label{app:moco-inexact}
In this section we provide the assumptions used for obtaining the results in Lemmas \ref{lemma:hkm error} and \ref{lemma:hkm error-new}, and prove the respective results. Throughout the section, we write $\E[\cdot|\mathcal{F}_k]$ as $\E_k[\cdot]$ for conciseness. Consider the following conditions
\begin{enumerate}[label=(\alph*)]
\setlength\itemsep{-0.15em}
    \item\label{item:strong convex} For any $x \in \mathbb{R}^d$, $l_m(x,z)$ is strongly convex w.r.t. $z$ with modulus $\mu_m > 0$.
    \item\label{item:g lip 2} There exist constants $L_{xz}, l_{xz}, l_{zz}$ such that $\nabla_z l_m(x,z)$ is $L_{xz}$-Lipschitz continuous w.r.t. $x$; $\nabla_z l_m(x,z)$ is $L_{zz}$-Lipschitz continuous w.r.t. $z$. $\nabla_{xz}l_m(x,z)$, $\nabla_{zz} l_m(x,z)$ are respectively $l_{xz}$-Lipschitz and $l_{zz}$-Lipschitz continuous w.r.t. $(x,z)$.
    \item\label{item:f lip} There exist constants $l_{fx}, l_{fz}, l_{fz}', l_z$ such that $\nabla_x f_m(x,z)$ and $\nabla_z f_m(x,z)$ are respectively $l_{fx}$ and $l_{fz}$ Lipschitz continuous w.r.t. $z$; $\nabla_z f_m(x,z)$ is $l_{fz}'$-Lipschitz continuous w.r.t. $x$; $f_m(x,z)$ is $l_z$-Lipschitz continuous w.r.t. $z$.
     \item\label{item:moment2} There exist constants $C_x,C_{xz},C_z,C_l$ such that $\E_k\|\nabla_x f_m(x_k,z_{k,m},\xi_k)\|^2 \leq C_x^2$, $\E_k\|\nabla_{xz} l_m(x_k,z_{k,m},\varphi_k')\|^2 \leq C_{xz}^2$, $\E_k\|\nabla_z f_m(x_k,z_{k,m},\xi_k)\|^2 \leq C_z^2$ and $\E_k\|\nabla_{z} l_m(x_k,z_{k,m},\varphi_k)\|^2 \leq C_l^2$.
     \item \label{item:H} There exist constants $C_H$ and $\sigma_H$ such that
     \begin{enumerate}[label=(\roman*)]
         \item\label{item:H-1}  $\|\E_k[H_{k,m}^{zz}]-[\nabla_{zz}l_m(x_k,z_{k,m})]^{-1}\|^2 \leq C_H \frac{\alpha_K^2}{\beta_K^2}$ and $\E_k\|H_{k,m}^{zz}\|^2 \leq \sigma_H^2$, where $\frac{\alpha_K^2}{\beta_K^2} < 1$, or
         \item\label{item:H-2} $\|\E_k[H_{k,m}^{zz}]-[\nabla_{zz}l_m(x_k,z_{k,m})]^{-1}\|^2 \leq C_H \beta_k$ and $\E_k\|H_{k,m}^{zz}\|^2 \leq \sigma_H^2$, where $\beta_K < 1$.
     \end{enumerate}
\end{enumerate}
The above conditions are standard in the literature \citep{ghadimi2018approximation}. In particular, condition \ref{item:H} on $H_{k,m}^{zz}$ can be guaranteed by \citep[Algorithm 3]{ghadimi2018approximation}.

 Before presenting the proofs for Lemmas \ref{lemma:hkm error} and \ref{lemma:hkm error-new}, we first state and prove the following helping Lemma that is useful in the proof of \ref{lemma:hkm error-new}, which is similar to \citep[Lemma 1]{rakhlin2011making}.

\begin{Lemma}[Convergence of $z_{t, k, m}$]\label{lemma:z-convergence}
Consider the sequences generated by \eqref{eq:z update-new}. Under conditions \ref{item:strong convex}--\ref{item:moment2} and \ref{item:H}\ref{item:H-2}, and with $\eta_t = \frac{1}{\mu_m t}$, we have for any $m$ that
\begin{align}\label{eq:z-convergence}
    \E\|z_{t, k,m}-z_{k,m}^*\|^2 \leq \frac{4 C_\ell^2}{\mu^2_m t}.
\end{align}
\end{Lemma}
\begin{proof}

We start by obtaining a bound for $\E\|z_{1, k, m}-z_{k,m}^*\|^2$. From the $\mu_m$-strong convexity of $\ell_m$, we have 
\begin{equation}
    \langle z_{1, k,m}-z_{k,m}^*,\nabla_z l_m(x_k,z_{1, k,m}) \rangle \geq \frac{\mu_m}{2} \| z_{1, k,m}-z_{k,m}^* \|^2,
\end{equation}
which along with Cauchy-Schwartz inequality, gives
\begin{equation}\label{eq:z1-1}
    \| \nabla_z l_m(x_k,z_{1, k,m}) \|^2 \geq \frac{\mu_m^2}{4} \| z_{1, k,m}-z_{k,m}^* \|^2.
\end{equation}
Then, rearranging the above inequality and taking total expectation on both sides, it follows that
\begin{align}\label{eq:z1-3}
    \E\| z_{1, k,m}-z_{k,m}^* \|^2 
    &\leq \frac{4}{\mu_m^2}\E\| \nabla_z l_m(x_k,z_{1, k,m}) \|^2 \nonumber \\
    &= \frac{4}{\mu_m^2} \E\| \E \left[ \nabla_z l_m(x_k,z_{1, k,m}, \varphi_{1,k})) ~|~  x_k,z_{1, k,m}\right] \|^2 \nonumber \\
    &\leq \frac{4}{\mu_m^2} \E \left[ \E \left[ \|  \nabla_z l_m(x_k,z_{1, k,m}, \varphi_{1,k}))  \|^2~|~  x_k,z_{1, k,m}\right] \right] \nonumber \\ 
    &= \frac{4}{\mu_m^2} \E  \|  \nabla_z l_m(x_k,z_{1, k,m}, \varphi_{1,k}))  \|^2 \nonumber\\
    &\leq \frac{4}{\mu_m^2} C^2_\ell,
\end{align}
where the second inequality is due to Jensen's inequality and the last inequality is due to condition \ref{item:H}. Furthermore, we also have
\begin{align}\label{eq:zt-1}
    & \quad \E\|z_{t+1, k, m}-z_{k,m}^*\|^2 \nonumber \\
    &= \E\|z_{t, k, m}- \eta_t \nabla_z l_m(x_k,z_{k,m},\varphi_{t, k}) - \nabla z_{k,m}^*\|^2 \nonumber \\
    &= \E\|z_{t, k,m}-z_{k,m}^*\|^2 - 2\eta_t \E \langle z_{t, k,m}-z_{k,m}^*,\nabla_z l_m(x_k,z_{t, k,m},\varphi_{t, k}) \rangle + \eta_t^2 \E\|\nabla_z l_m(x_k,z_{k,m},\varphi_{t, k})\|^2 \nonumber\\ 
    &= \E\|z_{t, k,m}-z_{k,m}^*\|^2 - 2\eta_t \E\langle z_{t, k,m}-z_{k,m}^*,\nabla_z l_m(x_k,z_{t, k,m}\rangle + \eta_t^2 \E\|\nabla_z l_m(x_k,z_{k,m},\varphi_{t, k})\|^2 \nonumber\\
    &\leq \E\|z_{t, k,m}-z_{k,m}^*\|^2 - 2\eta_t \E\langle z_{t, k,m}-z_{k,m}^*,\nabla_z l_m(x_k,z_{t, k,m}) \rangle + \eta_t^2  C_\ell^2 \nonumber\\
    &\leq \E\|z_{t, k,m}-z_{k,m}^*\|^2 - 2\eta_t  \mu_m \E\| z_{t, k,m} - z_{k,m}^* \|^2 + \eta_t^2  C_\ell^2 \nonumber\\
    &= \left( 1 - 2\eta_t \mu_m \right) \E\|z_{t, k,m}-z_{k,m}^*\|^2 + \eta_t^2  C_\ell^2,
\end{align}
where the first equality is due to the nested loop update for $z_{t, k, m}$, the second equality follows from the fact that $\nabla_z l_m(x_k,z_{t, k,m}) = \E \left[ \nabla_z l_m(x_k,z_{t, k,m},\varphi_{t, k}) ~ | ~ x_k,  z_{t, k,m} \right]$,  the first inequality is due to condition \ref{item:moment2}, and the last inequality is due to condition \ref{item:strong convex}. Now, setting $\eta_t=\frac{1}{\mu_m t}$, we have
\begin{equation}\label{eq:zt-2}
    \E\|z_{t+1, k, m}-z_{k,m}^*\|^2 \leq \left(1 - \frac{2}{t} \right) \E\|z_{t, k,m}-z_{k,m}^*\|^2 + \frac{C_\ell^2}{\mu_m^2 t^2}.
\end{equation}
Then, we have the claim \eqref{eq:z-convergence} is true for $t=1$ from \eqref{eq:z1-3}. We also have the claim true for $t=2$, by \eqref{eq:zt-2}. Then, by considering $t\geq 3$ as the base case, we can show the claim \eqref{eq:z-convergence} holds for any $t\in[T]$.
\end{proof}

With the above result, we are ready to present the proofs for Lemmas \ref{lemma:hkm error} and \ref{lemma:hkm error-new}.

\subsection{Proof of Lemma \ref{lemma:hkm error}}\label{app:pf-lemma-1}

\begin{proof}

We first prove \eqref{eq:lem1-2} with
\begin{align}
    &\E_k\|h_{k,m}-\E_k[h_{k,m}]\|^2 \nonumber\\
    &=\E_k\|h_{k,m}\|^2 - \|\E_k[h_{k,m}]\|^2\leq \E_k\|h_{k,m}\|^2 \nonumber\\
    &\leq 2\E_k\|\nabla_x f_m(x_k,z_{k,m},\xi_k)\|^2 + 2\E_k \big[\|\nabla_{xz}^2 l_m(x_k,z_{k,m},\varphi_k')\|^2 \|H^{zz}_{k,m}\|^2 \|\nabla_z f_m(x_k,z_{k,m},\xi_k)\|^2 \big] \nonumber\\
    &\leq  2 C_x^2 + 2 C_{xz}^2 C_z^2 \sigma_{H}^2,
\end{align}
where the last inequality follows from item \ref{item:moment2} along with the independence of $\nabla_{xz}^2 l_m(x_k,z_{k,m},\varphi_k')$, $H^{zz}_{k,m}$ and $\nabla_z f_m(x_k,z_{k,m},\xi_k)$ given $\mathcal{F}_k$.

Next we start to prove \eqref{eq:lem1-1}. From \eqref{eq:z update2}, we have
\begin{align}
    \E_k[h_{k,m}] = \nabla_x f_m(x_k,z_{k,m}) -\nabla_{xz}^2 l_m(x_k,z_{k,m}) \E_k[H^{zz}_{k,m}] \nabla_z f_m(x_k,z_{k,m}).
\end{align}
In the following proof, we write $z_{k,m}^*(x_k)$ as $z_{k,m}^*$.
The above inequality along with \eqref{eq:nabla f_m(x)} implies
\begin{align}\label{eq:idk}
    &\| \E_k[h_{k,m}]-\nabla f_m(x_k) \| \nonumber\\
    &\leq  \| \nabla_x f_m(x_k,z_{k,m}) - \nabla_x f_m(x_k,z_{k,m}^*) \| \nonumber\\
    &~~~~+ \|\nabla_{xz}^2 l_m(x_k,z_{k,m}) \E_k[H^{zz}_{k,m}] \nabla_z f_m(x_k,z_{k,m}) - \nabla_{xz}^2 l_m(x_k,z_{k,m}) [\nabla_{zz}l_m(x_k,z_{k,m})]^{-1}  \nabla_z f_m(x_k,z_{k,m})\| \nonumber\\
    &~~~~+ \|\nabla_{xz}^2 l_m(x_k,z_{k,m}) [\nabla_{zz}l_m(x_k,z_{k,m})]^{-1}  \nabla_z f_m(x_k,z_{k,m})\nonumber\\
    &~~~~~~~~~~~~ - \nabla_{xz}^2 l_m(x_k,z_{k,m}^*) [\nabla_{zz}l_m(x_k,z_{k,m}^*)]^{-1}  \nabla_z f_m(x_k,z_{k,m}^*)\| \nonumber\\
    &\leq l_{fx}\|z_{k,m}-z_{m}^*(x_k)\| + L_{xz} l_z \|\E_k[H_{k,m}^{zz}]-[\nabla_{zz}l_m(x_k,z_{k,m})]^{-1}\| \nonumber\\
    &~~~~+ \big(\frac{L_{xz}l_z}{\mu_m} + \frac{l_{fz}}{\mu_m}+  \frac{l_{fz}l_z l_{zz}}{\mu_m^2} \big)\|z_{k,m}-z_{m}^*(x_k)\|
\end{align}
where the last inequality follows from conditions \ref{item:strong convex}--\ref{item:f lip}.
The inequality \eqref{eq:idk} implies
\begin{align}\label{eq:hkm-nablafmxk}
    \| \E_k[h_{k,m}]-\nabla f_m(x_k) \|^2 
    &\leq  2 L_{xz}^2 l_z^2 \|\E_k[H_{k,m}^{zz}]-[\nabla_{zz}l_m(x_k,z_{k,m})]^{-1}\|^2 \nonumber\\
    &~~+ 2\big(l_{fx}+\frac{L_{xz}l_z}{\mu_m} + \frac{l_{fz}}{\mu_m}+  \frac{l_{fz}l_z l_{zz}}{\mu_m^2}\big)^2\|z_{k,m}-z_{m}^*(x_k)\|^2 \nonumber\\
    &\leq 2 L_{xz}^2 l_z^2 C_H \frac{\alpha_K^2}{\beta_K^2} + 2(l_{fx}+\frac{L_{xz}l_z}{\mu_m} + \frac{l_{fz}}{\mu_m}+  \frac{l_{fz}l_z l_{zz}}{\mu_m^2})^2\|z_{k,m}-z_{m}^*(x_k)\|^2
\end{align}
where the last inequality follows from condition \ref{item:H}\ref{item:H-1} on the quality of $H_{k,m}^{zz}$.
Thus to prove \eqref{eq:lem1-1}, it suffices to prove $\frac{1}{K}\sum_{k=1}^K\E\|z_{k,m}-z_{k,m}^*\|^2  = \mathcal{O}\big(\frac{\alpha_K^2}{\beta_K^2}\big)$ next.

\emph{The convergence of $z_{k,m}$.} 
We start with
\begin{align}\label{eq:zk+1-zk+1*}
    \|z_{k+1,m}-z_{k+1,m}^*\|^2 = \|z_{k+1,m}-z_{k,m}^*\|^2 + 2 \langle z_{k+1,m}- z_{k,m}^*, z_{k,m}^*-z_{k+1,m}^*\rangle + \|z_{k,m}^*-z_{k+1,m}^*\|^2.
\end{align}
The first term is bounded as
\begin{align}\label{eq:zk+1-zk*_}
    &\E_k\|z_{k+1,m}-z_{k,m}^*\|^2 \nonumber\\
    &= \E_k\|z_{k,m}-\beta_k\nabla_z l_m(x_k,z_{k,m},\varphi_k)-z_{k,m}^*\|^2 \nonumber\\
    &= \|z_{k,m}-z_{k,m}^*\|^2 - 2\beta_k \langle z_{k,m}-z_{k,m}^*,\nabla_z l_m(x_k,z_{k,m}) \rangle + \beta_k^2 \E_k\|\nabla_z l_m(x_k,z_{k,m},\varphi_k)\|^2 \nonumber\\
    &\leq (1-2\mu_m\beta_k)\|z_{k,m}-z_{k,m}^*\|^2 + C_l^2\beta_k^2 
\end{align}
where the last inequality follows from condition \ref{item:strong convex} and \ref{item:moment2}.

Under conditions \ref{item:strong convex}--\ref{item:f lip}, it is shown that there exists a constant $L_{z,m}$ such that $z_m^*(x)$ is $L_{z,m}$-lipschitz continuous \citep[Lemma 2.2 (b)]{ghadimi2018approximation}. Let $L_z=\max_{m}L_{z,m}$, then the second term in \eqref{eq:zk+1-zk+1*} can be bounded as
\begin{align}\label{eq:zk+1-zk+1* I2}
    \langle z_{k+1,m}- z_{k,m}^*, z_{k,m}^*-z_{k+1,m}^*\rangle
    &\leq L_z\|z_{k+1,m}- z_{k,m}^*\|\|x_k-x_{k+1}\| \nonumber\\
    &\leq L_z C_y \alpha_k \|z_{k+1,m}- z_{k,m}^*\| \nonumber\\
    &\leq  \frac{\mu_m}{2}\beta_k \|z_{k+1,m}- z_{k,m}^*\|^2 +  \frac{1}{2}L_z^2 C_y^2 \mu_m^{-1} \frac{\alpha_k^2}{\beta_k}
\end{align}
where $C_y=\sup\|Y_k\|=\sup_{x\in\mathbb{R}^d}\|\nabla F(x)\| < \infty$ (by Assumption \ref{assumption:lip}), and the last inequality uses the Young's inequality. 

The last term in \eqref{eq:zk+1-zk+1*} is bounded as
\begin{align}\label{eq:zk+1-zk+1* I3}
    \|z_{k,m}^*-z_{k+1,m}^*\|^2 \leq L_z^2 C_y^2 \alpha_k^2.
\end{align}
Substituting \eqref{eq:zk+1-zk*_}--\eqref{eq:zk+1-zk+1* I3} into \eqref{eq:zk+1-zk+1*} yields
\begin{align}\label{eq:z final}
    \E_k\|z_{k+1,m}-z_{k+1,m}^*\|^2 
    &\leq (1-\mu_m\beta_k)\|z_{k,m}-z_{k,m}^*\|^2 + C_l^2\beta_k^2 +  L_z^2 C_y^2 \mu_m^{-1} \frac{\alpha_k^2}{\beta_k} + L_z^2 C_y^2 \alpha_k^2.
\end{align}

Taking total expectation on both sides, and then telescoping implies that (set $\alpha_k=\alpha_K, \beta_k=\beta_K,\,\forall k$)
\begin{align}
    \frac{1}{K}\sum_{k=1}^K\E\|z_{k,m}-z_{k,m}^*\|^2 = \mathcal{O}\big(\frac{1}{K \beta_K}\big) + \mathcal{O}\big(\beta_K\big) + \mathcal{O}\big(\alpha_K\big) + \mathcal{O}\big(\frac{\alpha_K^2}{\beta_K^2}\big).
\end{align}
The last inequality along with the step size choice in Lemma \ref{lemma:hkm error} implies that
\begin{align}
    \frac{1}{K}\sum_{k=1}^K\E\|z_{k,m}-z_{k,m}^*\|^2  = \mathcal{O}\big(\frac{\alpha_K^2}{\beta_K^2}\big),
\end{align}
which along with \eqref{eq:hkm-nablafmxk} implies \eqref{eq:lem1-1}. This completes the proof of Lemma \ref{lemma:hkm error}.
\end{proof}

\subsection{Proof of Lemma \ref{lemma:hkm error-new}}\label{app:pf-lemma-2}

\begin{proof}
The proof of \eqref{eq:lem2-2} follows exactly the same as in the proof of \eqref{eq:lem1-2} above. For proof of \eqref{eq:lem2-2}, we restate the inequality \eqref{eq:idk} obtained in the proof of \eqref{eq:lem1-1} as
\begin{align}\label{eq:idk-re}
    \| \E_k[h_{k,m}]-\nabla f_m(x_k) \| &\leq l_{fx}\|z_{k,m}-z_{m}^*(x_k)\| + L_{xz} l_z \|\E_k[H_{k,m}^{zz}]-[\nabla_{zz}l_m(x_k,z_{k,m})]^{-1}\| \nonumber\\
    &~~~~+ \big(\frac{L_{xz}l_z}{\mu_m} + \frac{l_{fz}}{\mu_m}+  \frac{l_{fz}l_z l_{zz}}{\mu_m^2} \big)\|z_{k,m}-z_{m}^*(x_k)\|.
\end{align}

The above inequality implies that
\begin{align}\label{eq:hkm-nablafmxk-new}
    \| \E_k[h_{k,m}]-\nabla f_m(x_k) \|^2 
    &\leq  2 L_{xz}^2 l_z^2 \|\E_k[H_{k,m}^{zz}]-[\nabla_{zz}l_m(x_k,z_{k,m})]^{-1}\|^2 \nonumber\\
    &~~+ 2\big(l_{fx}+\frac{L_{xz}l_z}{\mu_m} + \frac{l_{fz}}{\mu_m}+  \frac{l_{fz}l_z l_{zz}}{\mu_m^2}\big)^2\|z_{k,m}-z_{m}^*(x_k)\|^2 \nonumber\\
    &\leq 2 L_{xz}^2 l_z^2 C_H \beta_K+ 2(l_{fx}+\frac{L_{xz}l_z}{\mu_m} + \frac{l_{fz}}{\mu_m}+  \frac{l_{fz}l_z l_{zz}}{\mu_m^2})^2\|z_{k,m}-z_{m}^*(x_k)\|^2
\end{align}
where the last inequality follows from condition \ref{item:H}\ref{item:H-2} on the quality of $H_{k,m}^{zz}$.
Thus to prove \eqref{eq:lem2-1}, it suffices to prove $\E\|z_{k,m}-z_{k,m}^*\|^2  = \mathcal{O}\big(\beta_K\big)$ next. To this end, with $T=\mathcal{O}\left(\frac{1}{\beta_K}\right)$ and using Lemma \ref{lemma:z-convergence}, we have 
\begin{align}
    \E\|z_{k,m}-z_{k,m}^*\|^2 = \E\|z_{T, k,m}-z_{k,m}^*\|^2  \leq \frac{4C^2_\ell}{\mu_m T} = \mathcal{O}(\beta_K)
\end{align}

which along with \eqref{eq:hkm-nablafmxk-new} implies \eqref{eq:lem2-1}. This completes the proof of Lemma \ref{lemma:hkm error-new}.

\end{proof}

\subsection{Algorithm for MoCo with Nested Loop}

In this section we provide the omitted pseudo-code for MoCo with nested loop update described in Section \ref{sec:moco-inexact}. As remarked in Section \ref{sec:moco-inexact}, this algorithm can relate to the actor critic setting with multiple critics, with multiple critic update steps at each iteration.

\begin{algorithm}[h!]
	\caption{MoCo with inexact gradient}\label{alg:MoCo-inexact} 
	\begin{algorithmic}[1]
	    \State \textbf{Input} Initial model parameter $x_0$, tracking parameters $\{y_{0, i}\}_{m=1}^M$, lower level parameter $\{z_{0, i}\}_{m=1}^M$,  convex combination coefficient parameter $\lambda_0$, and their respective learning rates $\{\alpha_k\}_{k=0}^{K}$,  and $\{\beta_k\}_{k=0}^{K}$, $\{\gamma_k\}_{k=0}^{K}$.
	    \For {$k=0, \dots, K-1$}
	        \For {objective $m=1, \dots, M$}
                \State Set $z_{1, k, m} \coloneqq z_{k-1, m}$
                \For {$t=1, \dots, T$}
	               \State Update $z_{t+1, k, m}$ following  \eqref{eq:z update-new}
                \EndFor
                \State Set $z_{k, m} \coloneqq z_{T, k, m}$
	            \State Obtain $h_{k, m}$ following \eqref{eq:z update2}
	            \State Update $y_{k+1,m}$ following \eqref{eq:y update}
	        \EndFor
	    \State Update $\lambda_{k+1}$ and $x_{k+1}$ following  \eqref{eq:lambda update}-\eqref{eq:x update}
	    \EndFor
	    \State \textbf{Output} $x_K$
	\end{algorithmic} 
\end{algorithm}

\section{Proof of Lemma \ref{lemma:lower}}\label{app:pf-lemma-3}
Before we present the main proof, we first introduce the Lemmas \ref{lemma:lambda* lip} and \ref{lemma:frho subopt}, which are  direct consequences of \citep[Theorem 2F.7]{dontchev2009implicit} and \citep[Lemma A.1]{koshal2011multiuser}, respectively.
\begin{Lemma}\label{lemma:lambda* lip}
Under Assumption \ref{assumption:lip} , there exists a constant $L_\lambda \coloneqq \rho^{-1} \sum_{m=1}^ML_m$ such that the following inequality holds
\begin{align}\label{eq:lambda* lip}
    \|\lambda^*_\rho(x)-\lambda^*_\rho(x')\| \leq L_\lambda \|\nabla F(x) - \nabla F(x') \|,
\end{align}
which further indicates
\begin{align}
  \|\lambda^*_\rho(x)-\lambda^*_\rho(x')\| \leq L_{\lambda,F} \|x - x' \|,
\end{align}
where $L_{\lambda,F} \coloneqq \rho^{-1} L$, with $L= \left(\sum\limits_{m=1}^M L_{m}\right)\left(\sum\limits_{m=1}^M L_{m,1}\right)$.
\end{Lemma}

\begin{Lemma}\label{lemma:frho subopt}
For any $\rho>0$ and $x\in\mathbb{R}^d$, we have
\begin{align}
    0\leq \|\nabla F(x) \lambda^*_\rho(x)\|^2 - \|\nabla F(x)\lambda^*(x)\|^2 \leq \frac{\rho}{2}\left(1-\frac{1}{M}\right).
\end{align}
\end{Lemma}

Now we start to prove Lemma \ref{lemma:lower}.
\begin{proof}
Throughout the following proof, we write $\E[\cdot|\mathcal{F}_k]$ as $\E_k[\cdot]$ for conciseness.
With $d(x_k)=\nabla F(x_k) \lambda^*(x_k)$, we have
\begin{align}\label{eq:multi-grad-dp}
    &\|d(x_k) - Y_k \lambda_k\|^2 \nonumber \\ &= \|\nabla F(x_k) \lambda^*(x_k) - \nabla F(x_k)\lambda^*_\rho(x_k) + \nabla F(x_k)\lambda^*_\rho(x_k) - Y_k\lambda^*_\rho(x_k) + Y_k\lambda^*_\rho(x_k) - Y_k\lambda_k \|^2 \nonumber \\
    &\leq 3\| \nabla F(x_k)(x_k) \lambda^*(x_k) - \nabla F(x_k)\lambda^*_\rho(x_k)\|^2 + 3\|\nabla F(x_k)\lambda^*_\rho(x_k) - Y_k \lambda^*_\rho(x_k) \|^2 + 3\|Y_k\lambda^*_\rho(x_k) - Y_k\lambda_k  \|^2 \nonumber\\
    &\leq 3\| \nabla F(x_k) \lambda^*(x_k) - \nabla F(x_k)\lambda^*_\rho(x_k)\|^2 +  3\|\nabla F(x_k) - Y_k\|^2 +  3C_y\|\lambda^*_\rho(x_k) - \lambda_k \|^2.
\end{align}

From \eqref{eq:multi-grad-dp}, to prove Lemma \ref{lemma:lower}, it suffices to show $\| \nabla F(x_k) \lambda^*(x_k) - \nabla F(x_k)\lambda^*_\rho(x_k)\|$ diminishes, and establish the convergence of $Y_k$ and $\lambda_k$.

\textbf{Bounding $\| \nabla F(x_k) \lambda^*(x_k) - \nabla F\lambda^*_\rho(x_k)\|$.} Denoting $\lambda^*(x_k)$ as $\lambda^*_k$ and $\lambda^*_\rho(x_k)$ as $\lambda^*_{\rho,k}$, we first consider the following bound
\begin{align}\label{eq:flmab*-flambhat}
    \| \nabla F(x_k) \lambda^*_k - \nabla F(x_k)\lambda^*_{\rho,k} \|^2 &= \| \nabla F(x_k) \lambda^*_k \|^2 + \|\nabla F(x_k)\lambda^*_{\rho,k} \|^2 - 2\langle \nabla F(x_k) \lambda^*_k , \nabla F(x_k)\lambda^*_{\rho,k} \rangle \nonumber \\
    &\leq \|\nabla F(x_k)\lambda^*_{\rho,k} \|^2 - \| \nabla F(x_k) \lambda^*_k \|^2   \nonumber \\
    &\leq \frac{\rho}{2},
\end{align}
where the first inequality is due to the optimality condition 
$$\langle\lambda, \nabla F(x_k)^\top \nabla F(x_k) \lambda_k^* \rangle \geq \langle \lambda_k^* , \nabla F(x_k)^\top \nabla F(x_k) \lambda_k^* \rangle = \| \nabla F(x_k) \lambda_k^*\|^2 $$
for any $\lambda \in \Delta^M$, and the last inequality is due to Lemma \ref{lemma:frho subopt}. With the choice of $\rho = \Theta(K^{-\frac{1}{6}})$ as required by Theorem \ref{theorem:upper}, we have 
\begin{equation}\label{eq:rho final}
    \| \nabla F(x_k) \lambda^*_k - \nabla F(x_k)\lambda^*_{\rho,k} \|^2 = \mathcal{O}(K^{-\frac{1}{6}}).
\end{equation}

\textbf{Convergence of $Y_k$.} With $\E_k[h_{k,m}] = \E[h_{k,m}|\mathcal{F}_k]$, we start by
\begin{align}\label{eq:yk+1-yk*^2}
    \E_k\|y_{k+1,m}\!-\!\nabla f_m (x_{k+1})\|^2 
    &= \E_k\|y_{k+1,m}\!-\!\nabla f_m (x_k)\|^2 + 2 \E_k\langle y_{k+1,m}\!-\! \nabla f_m (x_k), \nabla f_m (x_k)\!-\!\nabla f_m (x_{k+1})\rangle \nonumber\\
    &+ \E_k\|\nabla f_m (x_k)-\nabla f_m (x_{k+1})\|^2.
\end{align}
 We bound the first term in \eqref{eq:yk+1-yk*^2} as
\begin{align}\label{eq:yk+1-yk*^2 I1-1}
&\E_k\|y_{k+1,m}-\nabla f_m (x_k)\|^2 \nonumber\\
&\leq  \E_k\|y_{k,m}-\beta_k\big(y_{k,m}-h_{k,m}\big)-\nabla f_m (x_k)\|^2 \nonumber\\
&=\|y_{k,m}-\nabla f_m (x_k)\|^2 -2\beta_k\langle y_{k,m}-\nabla f_m (x_k), y_{k,m}-\E_k[h_{k,m}]\rangle + \beta_k^2 \E_k\|y_{k,m}-h_{k,m}\|^2 \nonumber\\
&=\|y_{k,m}-\nabla f_m (x_k)\|^2 -2\beta_k\langle y_{k,m}-\nabla f_m (x_k), y_{k,m}-\nabla f_m (x_k)\rangle \nonumber\\
&~~~~-2\beta_k\langle y_{k,m}-\nabla f_m (x_k), \nabla f_m (x_k)-\E_k[h_{k,m}]\rangle + \beta_k^2 \E_k\|y_{k,m}-h_{k,m}\|^2 \nonumber\\
&\leq (1-2\beta_k)\|y_{k,m}-\nabla f_m (x_k)\|^2 + 2\beta_k \| y_{k,m}-\nabla f_m (x_k)\|\|\nabla f_m (x_k)-\E_k[h_{k,m}]\| + \beta_k^2 \E_k\|y_{k,m}-h_{k,m}\|^2 \nonumber\\
&\leq (1-\beta_k)\|y_{k,m}-\nabla f_m (x_k)\|^2 +  \beta_k\|\nabla f_m (x_k)-\E_k[h_{k,m}]\|^2 +\beta_k^2 \E_k\|y_{k,m}-h_{k,m}\|^2
\end{align}
where the last inequality follows from young's inequality. Now, consider the last term of \eqref{eq:yk+1-yk*^2 I1-1}. Selecting $\beta_k$ such that $3\beta_k^2 \leq \beta_k/2$, and Assumption \ref{assumption:h}, we have
\begin{align}\label{eq:yk+1-yk*^2 I1-2}
     &\beta_k^2\E_k\|y_{k,m}-h_{k,m}\|^2 \nonumber\\
     &=\beta_k^2\E_k\|y_{k,m}-\nabla f_m (x_k)+\nabla f_m (x_k)-\E_k[h_{k,m}]+\E_k[h_{k,m}]-h_{k,m}\|^2 \nonumber\\
     &\leq 3\beta_k^2\|y_{k,m}-\nabla f_m (x_k)\|^2 +3\beta_k^2\|\nabla f_m (x_k)-\E_k[h_{k,m}]\|^2 + 3 \beta_k^2 \E_k\|\E_k[h_{k,m}]-h_{k,m}\|^2\nonumber\\
     &\leq \frac{\beta_k}{2}\|y_{k,m}-\nabla f_m (x_k)\|^2 +3\beta_k^2\|\nabla f_m (x_k)-\E_k[h_{k,m}]\|^2+3\sigma_m^2\beta_k^2 .
\end{align}
Then, plugging in \eqref{eq:yk+1-yk*^2 I1-2} in \eqref{eq:yk+1-yk*^2 I1-1}, we obtain 
\begin{align}\label{eq:yk+1-yk*^2 I1}
    \E_k\|y_{k+1,m}-\nabla f_m (x_k)\|^2 &\leq (1-\frac{1}{2}\beta_k)\|y_{k,m}-\nabla f_m (x_k)\|^2 +\beta_k\|\nabla f_m (x_k)-\E_k[h_{k,m}]\|^2\nonumber\\
&~~~~+3\beta_k^2\|\nabla f_m (x_k)-\E_k[h_{k,m}]\|^2+3\sigma_m^2\beta_k^2. 
\end{align}

The second term in \eqref{eq:yk+1-yk*^2} can be bounded as
\begin{align}\label{eq:yk+1-yk*^2 I2}
     \langle y_{k+1,m}- \nabla f_m (x_k),& \nabla f_m (x_k)-\nabla f_m (x_{k+1})\rangle\nonumber\\ 
     &\leq \|y_{k+1,m}- \nabla f_m (x_k)\| \|\nabla f_m (x_k)-\nabla f_m (x_{k+1})\| \nonumber\\
     &\leq L_{m,1}\|y_{k+1,m}- \nabla f_m (x_k)\| \|x_{k+1}-x_k\| \nonumber\\
     &\leq L_{m,1}C_y \alpha_k \|y_{k+1,m}- \nabla f_m (x_k)\| \nonumber\\
     &\leq \frac{1}{8}\beta_k \|y_{k+1,m}-\nabla f_m (x_k)\|^2 + 2 L_{m,1}^2 C_y^2 \frac{\alpha_k^2}{\beta_k} \nonumber\\
     &\leq \frac{1}{8}\beta_k \|y_{k+1,m}-\nabla f_m (x_k)\|^2 + 2\Bar{L}_1^2 C_y^2 \frac{\alpha_k^2}{\beta_k}
\end{align}
where the second last inequality follows from young's inequality, and the last inequality follows from the definition $\Bar{L}_1=\max_m L_{m,1}$.

The last term in \eqref{eq:yk+1-yk*^2} can be bounded as
\begin{align}\label{eq:yk+1-yk*^2 I3}
    \|\nabla f_m (x_k)-\nabla f_m (x_{k+1})\|^2 \leq \Bar{L}_1^2 \|x_{k+1}-x_k\|^2 \leq \Bar{L}_1^2 C_y^2 \alpha_k^2.
\end{align}
Collecting the upper bounds in \eqref{eq:yk+1-yk*^2 I1}--\eqref{eq:yk+1-yk*^2 I3} and substituting them into \eqref{eq:yk+1-yk*^2} gives
\begin{align}\label{eq:y final}
   \E_k\|y_{k+1,m}-\nabla f_m (x_{k+1})\|^2 &\leq  (1-\frac{1}{4}\beta_k)\|y_{k,m}-\nabla f_m (x_k)\|^2 +(\beta_k+3\beta_k^2)\|\nabla f_m (x_k)-\E_k[h_{k,m}]\|^2\nonumber\\
   &~~+ 3\sigma_m^2\beta_k^2 + \Bar{L}_1^2 C_y^2 \alpha_k^2 + 4\Bar{L}_1^2 C_y^2 \frac{\alpha_k^2}{\beta_k}.
\end{align}
Suppose $\alpha_k=\alpha_K$ and $\beta_k=\beta_K$ are constants given $K$. Taking total expectation and then telescoping both sides of \eqref{eq:y final} gives
\begin{align}\label{eq: y not final}
    \frac{1}{K}\sum_{k=1}^K\E\|y_{k,m}-\nabla f_m (x_k)\|^2 
    &= \mathcal{O}\big(\frac{1}{K\beta_K} \big) + \mathcal{O}\big(\frac{1}{K}\sum_{k=1}^K \E\|\nabla f_m (x_k)-\E_k[h_{k,m}]\|^2 \big) \nonumber\\
    &~~~+\mathcal{O}(\beta_K) + \mathcal{O}\big(\frac{\alpha_K^2}{\beta_K^2}\big)
\end{align}
Along with the choice of step sizes as required by Theorem \ref{theorem:upper}, and due to Assumption \ref{assumption:h}, the last inequality gives 
\begin{equation}
    \frac{1}{K}\sum_{k=1}^K\E\|y_{k,m}-\nabla f_m (x_k)\|^2 = \mathcal{O}\big( K^{-\frac{1}{2}}\big)
\end{equation}
which, based on the definitions of $Y_k$ and $\nabla F(x_k)$, implies that
\begin{align}\label{eq:y final final}
    \frac{1}{K}\sum_{k=1}^K\E\|Y_k-\nabla F(x_k)\|^2 = \mathcal{O}\Big( M K^{-\frac{1}{2}}\Big).
\end{align}

\textbf{Convergence of $\lambda_k$.} We write $\lambda^*_\rho(x_k)$ in short as $\lambda_{\rho, k}^*$ in the following proof. We start by
\begin{align}\label{eq:lambdak+1-lambdak+1*}
    \|\lambda_{k+1}-\lambda_{\rho, k+1}^*\|^2 = \|\lambda_{k+1}-\lambda_{\rho, k}^*\|^2 + 2 \langle \lambda_{k+1}- \lambda_{\rho, k}^*, \lambda_{\rho, k}^*-\lambda_{\rho, k+1}^*\rangle + \|\lambda_{\rho, k}^*-\lambda_{\rho, k+1}^*\|^2.
\end{align}
The first term is bounded as
\begin{align}\label{eq:lambdak+1-lambdak*}
    \|\lambda_{k+1}-\lambda_{\rho, k}^*\|^2
    &= \| \Pi_{\Delta^M}\left(\lambda_k - \gamma_k \left( Y_k^\top Y_k + \rho I\right) \lambda_k \right)-\lambda_{\rho, k}^* \|^2 \nonumber \\
    &\leq \|\lambda_k-\gamma_k \left( Y_k^\top Y_k + \rho I\right) \lambda_k-\lambda_{\rho, k}^*\|^2 \nonumber\\
    &= \|\lambda_k-\lambda_{\rho, k}^*\|^2 - 2\gamma_k \langle\lambda_k-\lambda_{\rho, k}^*, \left( Y_k^\top Y_k + \rho I\right)\lambda_k \rangle + \gamma_k^2 \| \left( Y_k^\top Y_k + \rho I\right)\lambda_k\|^2 \nonumber\\
    &\leq \|\lambda_k-\lambda_{\rho, k}^*\|^2 - 2\gamma_k \langle\lambda_k-\lambda_{\rho, k}^*, \left( Y_k^\top Y_k + \rho I\right) \lambda_k \rangle +(C^2_y + \rho)^2\gamma_k^2 
\end{align}

Consider the second term in the last inequality:
\begin{align}\label{eq:lambd-inner-prod}
    &\langle\lambda_k-\lambda_{\rho, k}^*, \left(Y_k^\top Y_k +\rho I \right) \lambda_k \rangle \nonumber\\
    &=\langle\lambda_k-\lambda_{\rho, k}^*, (Y_k^\top Y_k-\nabla F(x_k)^\top \nabla F(x_k)) \lambda_k \rangle +\langle\lambda_k-\lambda_{\rho, k}^*, \left( \nabla F(x_k)^\top \nabla F(x_k) + \rho I \right) \lambda_k \rangle\nonumber\\
    &\geq -2C_y\|\lambda_k-\lambda_{\rho, k}^*\|\|  Y_k-\nabla F(x_k)\| +\langle\lambda_k-\lambda_{\rho, k}^*, \left(\nabla F(x_k)^\top \nabla F(x_k) +\rho I\right) (\lambda_k-\lambda_{\rho, k}^*)\rangle \nonumber\\
    &~~+ \langle\lambda_k-\lambda_{\rho, k}^*, \left(\nabla F(x_k)^\top \nabla F(x_k) +\rho I\right) \lambda_{\rho, k}^*\rangle\nonumber\\
    &\geq -2C_y\|\lambda_k-\lambda_{\rho, k}^*\|\|Y_k-\nabla F(x_k)\| +\rho\|\lambda_k-\lambda_{\rho, k}^*\|^2 \nonumber\\
    &\geq -2C_y^2 \rho^{-1}\|Y_k-\nabla F(x_k)\|^2 +\frac{\rho}{2}\|\lambda_k-\lambda_{\rho, k}^*\|^2
\end{align}
where the second last inequality follows from the optimality condition that 
$$\langle\lambda_k-\lambda_{\rho, k}^*, \left(\nabla F(x_k)^\top \nabla F(x_k)+\rho I\right) \lambda_{\rho, k}^* \rangle \geq 0,$$ 
and the last inequality in follows from the Young's inequality. 

Plugging in \eqref{eq:lambd-inner-prod} back to \eqref{eq:lambdak+1-lambdak*} gives
\begin{align}\label{eq:lambdak+1-lambdak*_}
     \|\lambda_{k+1}-\lambda_{\rho, k}^*\|^2 = (1-\rho\gamma_k)\|\lambda_k-\lambda_{\rho, k}^*\|^2 + 4 C_y^2 \rho^{-1}\gamma_k \|Y_k-\nabla F(x_k)\|^2 + (C_y^2 + \rho)^2\gamma_k^2.
\end{align}
With Lemma \ref{lemma:lambda* lip}, the second term in \eqref{eq:lambdak+1-lambdak+1*} can be bounded as
\begin{align}\label{eq:lambdak+1-lambdak+1* I2}
    \langle \lambda_{k+1}- \lambda_{\rho, k}^*, \lambda_{\rho, k}^*-\lambda_{\rho, k+1}^*\rangle
    &\leq L_{\lambda,F}\|\lambda_{k+1}- \lambda_{\rho, k}^*\|\|x_k-x_{k+1}\| \nonumber\\
    &\leq L_{\lambda,F} C_y \alpha_k \|\lambda_{k+1}- \lambda_{\rho, k}^*\| \nonumber\\
    &\leq  \frac{\rho}{4}\gamma_k \|\lambda_{k+1}- \lambda_{\rho, k}^*\|^2 +  L_{\lambda,F}^2 C_y^2\rho^{-1} \frac{\alpha_k^2}{\gamma_k},
\end{align}
where the last inequality is due to Young's inequality. The last term in \eqref{eq:lambdak+1-lambdak+1*} is bounded as
\begin{align}\label{eq:lambdak+1-lambdak+1* I3}
    \|\lambda_{\rho, k}^*-\lambda_{\rho, k+1}^*\|^2 \leq L_{\lambda,F}^2 C_y^2\alpha_k^2.
\end{align}
Substituting \eqref{eq:lambdak+1-lambdak*_}--\eqref{eq:lambdak+1-lambdak+1* I3} into \eqref{eq:lambdak+1-lambdak+1*} yields
\begin{align}\label{eq:lambda final}
    \|\lambda_{k+1}-\lambda_{\rho, k+1}^*\|^2 
    &\leq (1-\frac{\rho}{2}\gamma_k)\|\lambda_k-\lambda_{\rho, k}^*\|^2 + 4 C_y^2 \rho^{-1}\gamma_k \|Y_k-\nabla F(x_k)\|^2 + (C_y^2+\rho)^2\gamma_k^2 \nonumber\\
    &~~+ 2 L_{\lambda,F}^2 C_y^2 \rho^{-1} \frac{\alpha_k^2}{\gamma_k} + L_{\lambda,F}^2 C_y^2 \alpha_k^2.
\end{align}
Suppose $\alpha_k=\alpha_K$, $\beta_k=\beta_K$, and $\gamma_k=\gamma_K$ are constants given $K$. Taking total expectation, rearranging and taking telescoping sum on both sides of the last inequality gives
\begin{align}\label{eq:lambda final - 2}
    \frac{1}{K}\sum_{k=1}^K\E\|\lambda_{k}-\lambda_{\rho, k}^*\|^2 
    = \mathcal{O}\left(\frac{1}{K\rho\gamma_K} \right) + \mathcal{O}\left(\frac{1}{\rho^2 K}\sum_{k=1}^K \E\|Y_k - \nabla F(x_k)\|^2 \right)+ \mathcal{O}\left(\frac{\gamma_K}{\rho}\right)\nonumber\\ +\mathcal{O}\left(\frac{\alpha_K^2}{\gamma_K^2\rho^4}\right)+\mathcal{O}\left(\frac{\alpha_K^2}{\gamma_K\rho^3}\right)
\end{align}
where we have used $L_{\lambda, F} = \mathcal{O}(\frac{1}{\rho})$ from Lemma \ref{lemma:lambda* lip}. Then, plugging in the choices $\alpha_K = \Theta(K^{-\frac{3}{4}})$, $\beta_K=\Theta(K^{-\frac{1}{2}})$, $\gamma_K=\Theta(K^{-\frac{1}{3}})$,  $\rho = \Theta(K^{-\frac{1}{6}})$ and substituting from \eqref{eq:y final final} in \eqref{eq:lambda final - 2} gives
\begin{align}\label{eq:lambda final final} 
    \frac{1}{K}\sum_{k=1}^K\E\|\lambda_{k}-\lambda_{\rho, k}^*\|^2 = \mathcal{O}\big( M K^{-\frac{1}{6}}\big).
\end{align}
Thus, from \eqref{eq:multi-grad-dp}, \eqref{eq:rho final}, \eqref{eq:y final final}, and \eqref{eq:lambda final final}, we have
\begin{equation}
    \frac{1}{K}\sum_{k=1}^K\E \|d(x_k) - Y_k \lambda_k\|^2=  \mathcal{O}\big( M K^{-\frac{1}{6}}\big).
\end{equation}
This completes the proof. 
\end{proof}

\section{Proof of Theorem \ref{theorem:upper}} \label{app:pf-thm-2}
Before we go into the main proof, we first show the following key lemma.
\begin{Lemma}\label{lemma:descent direction}
For any $x \in \mathbb{R}^d$ and $\lambda \in \Delta^M$, with $d(x) = \nabla F(x) \lambda^*(x)$, it holds that
\begin{align}
    \langle d(x),\nabla F(x) \lambda \rangle \geq \|d(x)\|^2.
\end{align}

\end{Lemma}
\begin{proof}
We write $d_\lambda (x) = \nabla F(x)\lambda$ in the following proof.
Since $\Delta^M$ is a convex set, for any $\lambda' \in \Delta^M$, we have $\alpha (\lambda'-\lambda^*) + \lambda^* \in \Delta^M$ for any $\alpha \in [0,1]$. Then by $d(x)= \arg\min_{\lambda \in \Delta^M}\|d_\lambda (x)\|^2$, we have
\begin{align}
    \|d(x)\|^2 \leq \|\alpha (d_{\lambda'}(x)-d(x)) + d(x) \|^2.
\end{align}
Expanding the right hand side of the inequality gives
\begin{align}
     \alpha^2 \|d_{\lambda'}(x)-d(x)\|^2 + \alpha \langle d(x), d_{\lambda'}(x)-d(x)\rangle\geq 0.
\end{align}
Since this needs to hold for $\alpha$ arbitrarily close to $0$, we have
\begin{align}
    \langle d(x), d_{\lambda'}(x)-d(x)\rangle \geq 0, ~~~\forall \lambda' \in \Delta^M
\end{align}
which indicates the result inequality by rearranging.
\end{proof}

Now we can prove Theorem \ref{theorem:upper}.
\begin{proof}
By the $L_{m,1}$-smoothness of $f_m$, we have for any $m$,
\begin{align}\label{eq:fixk+1-fixk}
    f_m (x_{k+1}) 
    &\leq f_m (x_k) + \alpha_k \langle \nabla f_m(x_k), -Y_k \lambda_k \rangle + \frac{L_{m,1}}{2} \|x_{k+1}-x_k\|^2 \nonumber\\
    &\leq f_m (x_k) + \alpha_k \langle \nabla f_m(x_k), -Y_k \lambda_k \rangle + \frac{L_{m,1}}{2} C_y^2 \alpha_k^2.
\end{align}
The second term in the last inequality can be bounded as
\begin{align}\label{eq:fixk+1-fixk I2}
    &\langle \nabla f_m(x_k), -Y_k \lambda_k \rangle  \nonumber \\
    &= \langle \nabla f_m(x_k), \nabla F(x_k) \lambda_k^*-Y_k \lambda_k \rangle + \langle \nabla f_m(x_k), -\nabla F(x_k) \lambda_k^* \rangle \nonumber\\
    &\leq L_m\big( \|Y_k-\nabla F(x_k)\| + C_y \|\lambda_k-\lambda^*_{\rho,k}\| + \| \nabla F(x_k) \lambda_k^* - \nabla F(x_k)\lambda^*_{\rho,k}\| \big) + \langle \nabla f_m(x_k), -\nabla F(x_k) \lambda_k^* \rangle \nonumber\\
    &\leq L_m\big( \|Y_k-\nabla F(x_k)\| + C_y \|\lambda_k-\lambda^*_{\rho,k}\| + \| \nabla F(x_k) \lambda_k^* - \nabla F(x_k)\lambda^*_{\rho,k}\| \big) - \|\nabla F(x_k) \lambda_k^*\|^2,
\end{align}
where the last inequality follows from Lemma \ref{lemma:descent direction} by letting $\nabla F(x_k) \lambda=\nabla f_m(x_k)$, and the first inequality follows from
\begin{align}
    &\langle\nabla f_m(x_k), \nabla F(x_k) \lambda_k^*-Y_k \lambda_k \rangle \nonumber \\
    &\leq L_m\|\nabla F(x_k) \lambda_k^*-Y_k \lambda_k \| \nonumber\\
    &\leq L_m\|\nabla F(x_k) \lambda_k- Y_k \lambda_k + \nabla F(x_k)\lambda^*_{\rho,k}-\nabla F(x_k) \lambda_k + \nabla F(x_k) \lambda_k^* - \nabla F(x_k)\lambda^*_{\rho,k}\| \nonumber\\
    &\leq L_m\big( \|Y_k-\nabla F(x_k)\| + C_y \|\lambda_k-\lambda^*_{\rho,k}\| + \| \nabla F(x_k) \lambda_k^* - \nabla F(x_k)\lambda^*_{\rho,k}\| \big).
\end{align}
Plugging \eqref{eq:fixk+1-fixk I2} into \eqref{eq:fixk+1-fixk}, taking expectation on both sides and rearranging yields
\begin{align}
     \alpha_k\E\|\nabla F(x_k) \lambda_k^*\|^2 
     &\leq \E[f_m (x_{k}) -  f_m (x_{k+1})] + L_m\alpha_k\big( \E\|Y_k-\nabla F(x_k)\| + C_y \E\|\lambda_k-\lambda^*_{\rho,k}\| \nonumber \\  
     &~~~+ \| \nabla F(x_k) \lambda^*_k - \nabla F(x_k)\lambda^*_{\rho,k} \|\big) + \frac{L_{m,1}}{2} C_y^2 \alpha_k^2.
\end{align}
For all $k$, let $\alpha_k = \alpha_K$, $\beta_k=\beta_K$, and $\gamma_k=\gamma_K$ be constants given $K$. Then, taking telescope sum on both sides of the last inequality gives
\begin{align}
    \frac{1}{K}\sum_{k=1}^K \E\|\nabla F(x_k) \lambda_k^*\|^2  
    &\leq \frac{1}{\alpha_K K}(f_m(x_1) - \inf f_m(x)) + L_m  \frac{1}{K}\sum_{k=1}^K\big( \E\|Y_k-\nabla F(x_k)\| + C_y \E\|\lambda_k-\lambda^*_{\rho,k}\|\nonumber \\ &~~ + \| \nabla F(x_k) \lambda^*_k - \nabla F(x_k)\lambda^*_{\rho,k} \|\big)+ \frac{L_{m,1}}{2} C_y^2 \alpha_K \nonumber \\
    &= \mathcal{O}\left(\frac{1}{\alpha_K K}\right)  + \mathcal{O}\left(\frac{M^{\frac{1}{2}}}{K^{\frac{1}{2}}\beta_K^{\frac{1}{2}}} \right) +\mathcal{O}\left(M^{\frac{1}{2}}\beta_K^{\frac{1}{2}}\right) + \mathcal{O}\left(\frac{M^{\frac{1}{2}}\alpha_K}{\beta_K}\right) \nonumber\\
    &+ \mathcal{O}\left(\frac{1}{K^{\frac{1}{2}}\rho\gamma_K^{\frac{1}{2}}} \right) + \mathcal{O}\left(\frac{M^{\frac{1}{2}}}{\rho K^{\frac{1}{2}}\beta_K^{\frac{1}{2}}} \right) +\mathcal{O}\left(\frac{M^{\frac{1}{2}}\beta_K^{\frac{1}{2}}}{\rho} \right) + \mathcal{O}\left(\frac{M^{\frac{1}{2}}\alpha_K}{\rho \beta_K}\right) + \mathcal{O}\left(\frac{\gamma_K^{\frac{1}{2}}}{\rho^{\frac{1}{2}}}\right)  \nonumber \\
    &+\mathcal{O}\left(\frac{\alpha_K}{\gamma_K\rho^2}\right)+\mathcal{O}\left(\frac{\alpha_K}{\gamma_K^{\frac{1}{2}}\rho^{\frac{3}{2}}}\right) + \mathcal{O}\left( \rho^{\frac{1}{2}} \right) + \mathcal{O}(\alpha_K)
\end{align}

where the last equality is due to \eqref{eq:flmab*-flambhat}, \eqref{eq: y not final}, \eqref{eq:lambda final - 2}, along with Jensen's inequality for concave functions (i.e., for any concave function $g(x)$, $\E\left[g(x)\right] \leq g\left(\E[x]\right)$) and the subadditivity of square root function. Then, by choosing $\alpha_K = \Theta(K^{-\frac{3}{4}})$, $\beta_K=\Theta(K^{-\frac{1}{2}})$, $\gamma_K=\Theta(K^{-\frac{1}{3}})$,  and $\rho = \Theta(K^{-\frac{1}{6}})$, we arrive at $\frac{1}{K}\sum_{k=1}^K\E\|\nabla F(x_k) \lambda_k^*\|^2=\mathcal{O}(M^{\frac{1}{2}}K^{-\frac{1}{12}})$. 
\end{proof}

\section{Proof of Theorem \ref{theorem:upper-fbound}}
\begin{proof}
Recall from \eqref{eq:fixk+1-fixk}, by the $L_{m,1}$-smoothness of $f_m$, we have for any $m$,
\begin{align}
     f_m (x_{k+1}) 
    &\leq f_m (x_k) + \alpha_k \langle \nabla f_m(x_k), -Y_k \lambda_k \rangle + \frac{L_{m,1}}{2} C_y^2 \alpha_k^2.   
\end{align}
Multiplying both sides by $\lambda^m_k$ and summing over all $m\in[M]$, we obtain
\begin{align}\label{eq:fklambdak-init}
    F(x_{k+1}) \lambda_k \leq F(x_k)\lambda_k + \alpha_k \langle \nabla F(x_{k}) \lambda_k, -Y_k \lambda_k \rangle + \frac{\Bar{L}_1}{2} C_y^2 \alpha_k^2, 
\end{align}
where we have used $\lambda_k:=(\lambda^1_k, \lambda^2_k, \dots, \lambda^m_k, \dots, \lambda^M_k)^\top$ and $\Bar{L}_1=\max_m L_{m,1}$. We can bound the second term of \eqref{eq:fklambdak-init} as
\begin{align}\label{eq:inner-prod-nablafkyk}
    \langle \nabla F(x_k)\lambda_k, -Y_k\lambda_k \rangle 
    &= \langle \nabla F(x_k), -\nabla F(x_k) + \nabla F(x_k) - Y_k \rangle \nonumber \\
    &\leq -\|\nabla F(x_k)\lambda_k\|^2 + \frac{1}{2}\| \nabla F(x_k)\lambda_k\|^2 + \frac{1}{2}\| \nabla F(x_k) - Y_k \|^2 \nonumber \\
    &= -\frac{1}{2}\| \nabla F(x_k) \lambda_k\|^2 + \frac{1}{2}\| \nabla F(x_k) - Y_k \|^2,
\end{align}
where the first inequality is due to Cauchy-Schwartz and Young's inequalities. Substituting \eqref{eq:inner-prod-nablafkyk} in \eqref{eq:fklambdak-init} and rearranging, we have
\begin{align}
    \frac{\alpha_k}{2}\| \nabla F(x_k) \lambda_k \|^2 \leq F(x_k)\lambda_k - F(x_{k+1})\lambda_k + \frac{\alpha_k}{2}\| \nabla F(x_k) - Y_k \|^2 + \frac{\Bar{L}_1}{2} \alpha^2_k C_y^2.
\end{align}
For all $k$, let $\alpha_k = \alpha_K$, $\beta_k=\beta_K$, and $\gamma_k=\gamma_K$ be constants given $K$. We then take total expectation on both sides and sum over iterations to obtain
\begin{align} \label{eq:fklambdak-expect}
    \frac{\alpha_K}{2}\sum\limits_{k=1}^K \E\| \nabla F(x_k) \lambda_k \|^2 \leq \sum\limits_{k=1}^K \E \left[ F(x_k)\lambda_k - F(x_{k+1})\lambda_k \right] + \frac{\alpha_K}{2} \sum\limits_{k=1}^K \E \|\nabla F(x_k) - Y_k \|^2 + \frac{\Bar{L}_1}{2} \alpha_K^2 K C_y^2.
\end{align}
We bound the first term on the right-hand side of the inequality \eqref{eq:fklambdak-expect} as
\begin{align}\label{eq:fk+1lambdak-fklambdak}
    \sum\limits_{k=1}^K \E \left[ F(x_k)\lambda_k - F(x_{k+1})\lambda_k \right]
    &= \E\left[\sum\limits_{k=1}^{K-1}  F(x_{k+1}) (\lambda_{k+1} - \lambda_k) + F(x_1)\lambda_1 - F(x_{K+1})\lambda_k \right]\nonumber\\
    &\leq \E\left[\sum\limits_{k=1}^{K-1} \|F(x_{k+1})\|\| \lambda_{k+1} - \lambda_k \| + \|F(x_1)\|\|\lambda_1\| + \|F(x_{K+1})\|\|\lambda_k\| \right] \nonumber \\
    &\leq F\sum\limits_{k=1}^{K-1} \| \gamma_K Y_k^\top Y_k \lambda_k \| + 2 F \nonumber \\
    &\leq F C^2_y (K-1) \gamma_K + 2F,
\end{align}
where the first inequality is due to Cauchy-Schwartz, the second inequality is due to the bounds on $F(x_k)$, $\lambda_k$ and we have used the update for $\lambda_k$ for all $k\in[K]$ with $\rho=0$, and the last inequality is due to the bound on $Y_k$ and $\lambda_k$ for all $k\in[K]$. Substituting \eqref{eq:fk+1lambdak-fklambdak} in \eqref{eq:fklambdak-expect} and dividing both sides by $\frac{\alpha_K K}{2}$, we have
\begin{align} \label{eq:nablafk}
 \frac{1}{K}\sum\limits_{k=1}^K \E\| \nabla F(x_k) \lambda_k \|^2 
    &\leq 2 F C^2_y \frac{(K-1)}{K} \frac{\gamma_K}{\alpha_K} + 4F\frac{1}{\alpha_K K} + \frac{1}{K} \sum\limits_{k=1}^K \E \|\nabla F(x_k) - Y_k \|^2 + \Bar{L}_1\alpha_K C_y^2
\end{align}
which, along with \eqref{eq: y not final} and choosing $\alpha_K=\Theta(K^{-\frac{3}{5}})$, $\beta_K=\Theta(K^{-\frac{2}{5}})$, and $\gamma_K=\Theta(K^{-1})$, we obtain
\begin{equation}
 \frac{1}{K}\sum\limits_{k=1}^K \E\| \nabla F(x_k) \lambda_k \|^2 = \mathcal{O}(M K^{-\frac{2}{5}}).
\end{equation}
The result then follows by observing that for any $k\in[K]$, we have
\begin{equation}
    \|\nabla F(x_k)\lambda_k\|^2 \geq \min\limits_{\lambda}\|\nabla F(x_k)\lambda\|^2 = \|\nabla F(x_k)\lambda_k^*\|^2.
\end{equation}    
\end{proof}

\section{Proof of Theorem \ref{theorem:upper-fbound-stronger}}

\begin{proof}
\textbf{Convergence of $Y_k$.} We begin the proof by revisiting the convergence analysis on $Y_k$, under the assumptions considered in Theorem \ref{theorem:upper-fbound-stronger}. For convenience, we restate \eqref{eq:yk+1-yk*^2} here as
\begin{align}\label{eq:yk+1-yk*^2-new}
    \E_k\|y_{k+1,m}\!-\!\nabla f_m (x_{k+1})\|^2 
    &= \E_k\|y_{k+1,m}\!-\!\nabla f_m (x_k)\|^2 + 2 \E_k\langle y_{k+1,m}\!-\! \nabla f_m (x_k), \nabla f_m (x_k)\!-\!\nabla f_m (x_{k+1})\rangle \nonumber\\
    &+ \E_k\|\nabla f_m (x_k)-\nabla f_m (x_{k+1})\|^2.
\end{align}
 We bound the first term in \eqref{eq:yk+1-yk*^2-new} similar to  that in \eqref{eq:yk+1-yk*^2 I1}, as
\begin{align}\label{eq:yk+1-yk*^2 I1-new}
    \E_k\|y_{k+1,m}-\nabla f_m (x_k)\|^2 &\leq \Big(1-\frac{1}{2}\beta_k\Big)\|y_{k,m}-\nabla f_m (x_k)\|^2 +\beta_k\|\nabla f_m (x_k)-\E_k[h_{k,m}]\|^2\nonumber\\
&~~~~+3\beta_k^2\|\nabla f_m (x_k)-\E_k[h_{k,m}]\|^2+3\sigma_m^2\beta_k^2. 
\end{align}

The second term in \eqref{eq:yk+1-yk*^2-new} can be bounded as
\begin{align}\label{eq:yk+1-yk*^2 I2-new}
     \langle y_{k+1,m}- \nabla f_m (x_k),& \nabla f_m (x_k)-\nabla f_m (x_{k+1})\rangle\nonumber\\ 
     &\leq \|y_{k+1,m}- \nabla f_m (x_k)\| \|\nabla f_m (x_k)-\nabla f_m (x_{k+1})\| \nonumber\\
     &\leq L_{m,1}\|y_{k+1,m}- \nabla f_m (x_k)\| \|x_{k+1}-x_k\| \nonumber\\
     &\leq L_{m,1} \alpha_k \|y_{k+1,m}- \nabla f_m (x_k)\|\| Y_k\lambda_k \| \nonumber\\
     &\leq \frac{1}{8}\beta_k \|y_{k+1,m}-\nabla f_m (x_k)\|^2 + 2 L_{m,1}^2  \frac{\alpha_k^2}{\beta_k} \| Y_k\lambda_k \|^2 \nonumber\\
     &\leq \frac{1}{8}\beta_k \|y_{k+1,m}-\nabla f_m (x_k)\|^2 + 2\Bar{L}_1^2 \frac{\alpha_k^2}{\beta_k} \| Y_k\lambda_k \|^2
\end{align}
where the third inequality is due to the $x_k$ update, the second last inequality follows from young's inequality, and the last inequality follows from the definition $\Bar{L}_1=\max_m L_{m,1}$.

The last term in \eqref{eq:yk+1-yk*^2-new} can be bounded as
\begin{align}\label{eq:yk+1-yk*^2 I3-new}
    \|\nabla f_m (x_k)-\nabla f_m (x_{k+1})\|^2 \leq \Bar{L}_1^2 \|x_{k+1}-x_k\|^2 \leq \Bar{L}_1^2 \alpha_k^2 \| Y_k\lambda_k \|^2.
\end{align}
Collecting the upper bounds in \eqref{eq:yk+1-yk*^2 I1-new}--\eqref{eq:yk+1-yk*^2 I3-new} and substituting them into \eqref{eq:yk+1-yk*^2-new} gives
\begin{align}\label{eq:y final-new}
   \E_k\|y_{k+1,m}-\nabla f_m (x_{k+1})\|^2 &\leq  (1-\frac{1}{4}\beta_k)\|y_{k,m}-\nabla f_m (x_k)\|^2 +(\beta_k+3\beta_k^2)\|\nabla f_m (x_k)-\E_k[h_{k,m}]\|^2\nonumber\\
   &~~+ 3\sigma_m^2\beta_k^2 + \left( \Bar{L}_1^2 \alpha_k^2 + 4\Bar{L}_1^2  \frac{\alpha_k^2}{\beta_k}\right) \| Y_k\lambda_k \|^2.
\end{align}
For all $k$, let $\alpha_k = \alpha_K$, $\beta_k=\beta_K$, and $\gamma_k=\gamma_K$ be constants given $K$. Then, taking total expectation and then telescoping both sides of \eqref{eq:y final-new} gives
\begin{align}\label{eq: y not final-new-0}
    \frac{1}{K}\sum_{k=1}^K\E\|y_{k,m}-\nabla f_m (x_k)\|^2 
    &= \mathcal{O}\big(\frac{1}{K\beta_K} \big) + \mathcal{O}\big(\frac{1}{K}\sum_{k=1}^K \E\|\nabla f_m (x_k)-\E_k[h_{k,m}]\|^2 \big) \nonumber\\
    &~~~+\mathcal{O}(\beta_K) + \left( \Bar{L}_1^2 \frac{\alpha_K^2}{\beta_K} + 4\Bar{L}_1^2  \frac{\alpha_K^2}{\beta_K^2}\right) \frac{1}{K}\sum_{k=1}^K\E \| Y_k\lambda_k \|^2
\end{align}
Summing the last inequality over all objectives $m\in[M]$, and using Assumption \ref{assumption:h-new}, we obtain
\begin{align}\label{eq: y not final-new}
    \frac{1}{K}\sum_{k=1}^K\E\|Y_k-\nabla F (x_k)\|^2 
    &= \mathcal{O}\big(\frac{M}{K\beta_K} \big) + \mathcal{O}(M \beta_K) + \left( M\Bar{L}_1^2 \alpha_K^2 + 4 M \Bar{L}_1^2  \frac{\alpha_K^2}{\beta_K}\right) \frac{1}{K}\sum_{k=1}^K\E \| Y_k\lambda_k \|^2,
\end{align}
where we have used $\sum\limits_{m=1}^M \|y_{k,m}-\nabla f_m (x_k)\|^2 \geq \|Y_k-\nabla F (x_k)\|^2$. Then with the decomposition
\begin{equation}\label{eq:the-bound}
    \|Y_k\lambda_k\|^2 \leq 2\|Y_k - \nabla F(x_k)\|^2 + 2\| \nabla F(x_k)\lambda_k \|^2,
\end{equation}
we can arrive at 
\begin{align}\label{eq:y not final final-new}
    \frac{1}{K}\sum_{k=1}^K\E\|Y_k-\nabla F (x_k)\|^2 
    &= \mathcal{O}\big(\frac{M}{K\beta_K} \big) + \mathcal{O}(M \beta_K) + \left( 2M\Bar{L}_1^2 \alpha_K^2 + 8 M \Bar{L}_1^2  \frac{\alpha_K^2}{\beta_K}\right) \frac{1}{K}\sum_{k=1}^K\E \| \nabla F(x_k) \lambda_k \|^2 \nonumber\\
    &+\left( 2M\Bar{L}_1^2 \alpha^2 + 8 M \Bar{L}_1^2  \frac{\alpha_K^2}{\beta_K}\right) \frac{1}{K}\sum_{k=1}^K\E \|Y_k - \nabla F(x_k)\|^2.
\end{align}
Now, note that with the choice of stepsize $\alpha_K \leq \frac{1}{4 \bar{L} \sqrt{M}}\beta_K$, $~~0<\beta_K < 1$, and $\alpha_K$ and $\beta_K$ are on the same time scale, there exist some constant $0<C_1<1$ and valid choice of $\beta_K$ such that the following inequality holds
\begin{equation}\label{eq:ineq-c1}
    1 - 2M\Bar{L}_1^2 \alpha_K^2 - 8 M \Bar{L}_1^2  \frac{\alpha_K^2}{\beta_K} \geq C_1. 
\end{equation}
An example for a constant that satisfy the last inequality for aforementioned choice of $\alpha_K$ and $\beta_K$ is $C_1=\frac{1}{4}$. Then, from \eqref{eq:y not final final-new} and \eqref{eq:ineq-c1}, we can arrive at
\begin{equation}\label{eq:y final final-new}
    \frac{1}{K}\sum_{k=1}^K\E\|Y_k-\nabla F (x_k)\|^2 
    = \mathcal{O}\big(\frac{M}{K\beta_K} \big) + \mathcal{O}(M \beta_K) + \left( \frac{2M\Bar{L}_1^2}{C_1} \alpha_K^2 + \frac{8 M \Bar{L}_1^2}{C_1}  \frac{\alpha_K^2}{\beta_K}\right) \frac{1}{K}\sum_{k=1}^K\E \| \nabla F(x_k) \lambda_k \|^2
\end{equation}
Next, we analyse the $x_k$ sequence. For this purpose, we follow along similar lines as in the proof of Theorem \ref{theorem:upper-fbound}. Accordingly, from \eqref{eq:fixk+1-fixk}, we have for any $m$,
\begin{align}
     f_m (x_{k+1}) 
    &\leq f_m (x_k) + \alpha_k \langle \nabla f_m(x_k), -Y_k \lambda_k \rangle + \frac{L_{m,1}}{2} C_y^2 \alpha_k^2.   
\end{align}
Multiplying both sides by $\lambda^m_k$ and summing over all $m\in[M]$, we obtain
\begin{align}\label{eq:fklambdak-init-new}
    F(x_{k+1}) \lambda_k \leq F(x_k)\lambda_k + \alpha_k \langle \nabla F(x_{k}) \lambda_k, -Y_k \lambda_k \rangle + \frac{\Bar{L}_1}{2} C_y^2 \alpha_k^2, 
\end{align}
where we have used $\lambda_k:=(\lambda^1_k, \lambda^2_k, \dots, \lambda^m_k, \dots, \lambda^M_k)^\top$ and $\Bar{L}_1=\max_m L_{m,1}$. We can bound the second term of \eqref{eq:fklambdak-init-new} as
\begin{align}\label{eq:inner-prod-nablafkyk-new}
    \langle \nabla F(x_k)\lambda_k, -Y_k\lambda_k \rangle 
    &= \langle \nabla F(x_k)\lambda_k, -\nabla F(x_k)\lambda_k + \nabla F(x_k)\lambda_k - Y_k\lambda_k \rangle \nonumber \\
    &\leq -\|\nabla F(x_k)\lambda_k\|^2 + \langle \nabla F(x_k)\lambda_k, \nabla F(x_k)\lambda_k - Y_k\lambda_k \rangle \nonumber \\
    &\leq -\|\nabla F(x_k)\lambda_k\|^2 + \frac{1}{2}\| \nabla F(x_k)\lambda_k\|^2 + \frac{1}{2}\| \nabla F(x_k) - Y_k\lambda_k \|^2 \nonumber \\
    &\leq -\frac{1}{2}\| \nabla F(x_k) \lambda_k\|^2 + \frac{1}{2}\| \nabla F(x_k) - Y_k \|^2,
\end{align}
where the second inequality is due to Cauchy-Schwartz and Young's inequalities, and the last inequality is due to bound on $\lambda_k$.  Substituting \eqref{eq:inner-prod-nablafkyk-new} in \eqref{eq:fklambdak-init-new} and rearranging, we have
\begin{align}
    \frac{\alpha_k}{2}\| \nabla F(x_k) \lambda_k \|^2 \leq F(x_k)\lambda_k - F(x_{k+1})\lambda_k + \frac{\alpha_k}{2}\| \nabla F(x_k) - Y_k \|^2 + \frac{\Bar{L}_1}{2} \alpha^2 C_y^2.
\end{align}
For all $k$, let $\alpha_k = \alpha_K$, $\beta_k=\beta_K$, and $\gamma_k=\gamma_K$ be constants given $K$.  We then take total expectation on both sides and sum over iterations to obtain 
\begin{align} \label{eq:fklambdak-expect-new}
    \frac{\alpha_K}{2}\sum\limits_{k=1}^K \E\| \nabla F(x_k) \lambda_k \|^2 \leq \sum\limits_{k=1}^K \E \left[ F(x_k)\lambda_k - F(x_{k+1})\lambda_k \right] + \frac{\alpha_K}{2} \sum\limits_{k=1}^K \E \|\nabla F(x_k) - Y_k \|^2 + \frac{\Bar{L}_1}{2} \alpha_K^2 K C_y^2.
\end{align}
We bound the first term on the right-hand side of the inequality \eqref{eq:fklambdak-expect-new} as
\begin{align}\label{eq:fk+1lambdak-fklambdak-new}
    \sum\limits_{k=1}^K \E \left[ F(x_k)\lambda_k - F(x_{k+1})\lambda_k \right]
    &= \E\left[\sum\limits_{k=1}^{K-1}  F(x_{k+1}) (\lambda_{k+1} - \lambda_k) + F(x_1)\lambda_1 - F(x_{K+1})\lambda_k \right]\nonumber\\
    &\leq \E\left[\sum\limits_{k=1}^{K-1} \|F(x_{k+1})\|\| \lambda_{k+1} - \lambda_k \| + \|F(x_1)\|\|\lambda_1\| + \|F(x_{K+1})\|\|\lambda_k\| \right] \nonumber \\
    &\leq F\sum\limits_{k=1}^{K-1} \| \gamma_K Y_k^\top Y_k \lambda_k \| + 2 F \nonumber \\
    &\leq F C_y\sum\limits_{k=1}^{K} \gamma_K \|Y_k\lambda_k\|+ 2F,
\end{align}
where the first inequality is due to Cauchy-Schwartz, the second inequality is due to the bounds on $F(x_k)$, $\lambda_k$ and we have used the update for $\lambda_k$ for all $k\in[K]$ with $\rho=0$, and third inequality is due to the bound on $Y_k$ for all $k\in[K]$. Substituting \eqref{eq:fk+1lambdak-fklambdak-new} in \eqref{eq:fklambdak-expect-new} and dividing both sides by $\frac{\alpha_K K}{2}$, we have
\begin{align} \label{eq:nablafk-new}
 &\frac{1}{K}\sum\limits_{k=1}^K \E\| \nabla F(x_k) \lambda_k \|^2  \nonumber\\
    &\leq 2 F C_y \frac{1}{K} \frac{\gamma_K}{\alpha_K}\sum\limits_{k=1}^K \E \|Y_k\lambda_k\| + 4F\frac{1}{\alpha_K K} + \frac{1}{K} \sum\limits_{k=1}^K \E \|\nabla F(x_k) - Y_k \|^2 + \Bar{L}_1\alpha_K C_y^2 \nonumber\\
    &\leq  2F^2 C^2_y \frac{\gamma_K^2}{\alpha_K^2} + \frac{1}{4K} \sum\limits_{k=1}^K \E \|Y_k\lambda_k\|^2 + 4F\frac{1}{\alpha_K K} + \frac{1}{K} \sum\limits_{k=1}^K \E \|\nabla F(x_k) - Y_k \|^2 + \Bar{L}_1\alpha_K C_y^2 \nonumber\\
    &\leq 2F^2 C^2_y \frac{\gamma_K^2}{\alpha_K^2} + \frac{3}{2 K} \sum\limits_{k=1}^K \E \|\nabla F(x_k) - Y_k \|^2 + \frac{1}{2K}\sum\limits_{k=1}^K \E\| \nabla F(x_k) \lambda_k \|^2 + 4F\frac{1}{\alpha_K K} + \Bar{L}_1\alpha_K C_y^2 \nonumber\\
    &= \mathcal{O}\big(\frac{\gamma_K^2}{\alpha_K^2} \big) + \mathcal{O}\big(\frac{M}{K\beta_K} \big) + \mathcal{O}(M \beta_K) + \mathcal{O}(\frac{1}{\alpha_K K}) + \mathcal{O}(\alpha_K) \nonumber\\
    &+ \left( \frac{1}{2} + \frac{3M\Bar{L}_1^2}{C_1} \alpha_K^2 + \frac{12 M \Bar{L}_1^2}{C_1}  \frac{\alpha_K^2}{\beta_K}\right) \frac{1}{K}\sum_{k=1}^K\E \| \nabla F(x_k) \lambda_k \|^2
\end{align}
where the last equality is by substituting from \eqref{eq:y final final-new}. Now, with a similar argument that we made in \eqref{eq:ineq-c1}, given some $C_1$, there exist some constant $0<C_2<1$ such that 
\begin{equation}
    \frac{1}{2} - \frac{3M\Bar{L}_1^2}{C_1} \alpha_K^2 - \frac{12 M \Bar{L}_1^2}{C_1}  \frac{\alpha_K^2}{\beta_K} \geq C_2.
\end{equation}
Feasible choices of $C_1, C_2$ would be $C_1=C_2=\frac{1}{4}$. Then, we can have
\begin{align}
    \frac{1}{K}\sum\limits_{k=1}^K \E\| \nabla F(x_k) \lambda_k \|^2 = \mathcal{O}\big(\frac{\gamma_K^2}{\alpha_K^2} \big) + \mathcal{O}\big(\frac{M}{K\beta_K} \big) + \mathcal{O}(M \beta_K) + \mathcal{O}(\frac{1}{\alpha_K K}) + \mathcal{O}(\alpha_K),
\end{align}
which along with the choice of step sizes $\alpha_K=\Theta(K^{-\frac{1}{2}})$, $\beta_K=\Theta(K^{-\frac{1}{2}})$, $\gamma_K=\Theta(K^{-\frac{3}{4}})$, we arrive at
\begin{equation}
 \frac{1}{K}\sum\limits_{k=1}^K \E\| \nabla F(x_k) \lambda_k \|^2 = \mathcal{O}(M K^{-\frac{1}{2}}).
\end{equation}
The result then follows by observing that for any $k\in[K]$, we have
\begin{equation}
    \|\nabla F(x_k)\lambda_k\|^2 \geq \min\limits_{\lambda}\|\nabla F(x_k)\lambda\|^2 = \|\nabla F(x_k)\lambda_k^*\|^2.
\end{equation}    
\end{proof}

\begin{figure}[t]
     \centering
     \begin{subfigure}[b]{0.32\textwidth}
         \centering
         \includegraphics[trim={1cm 0 2cm 0},clip, width=\textwidth]{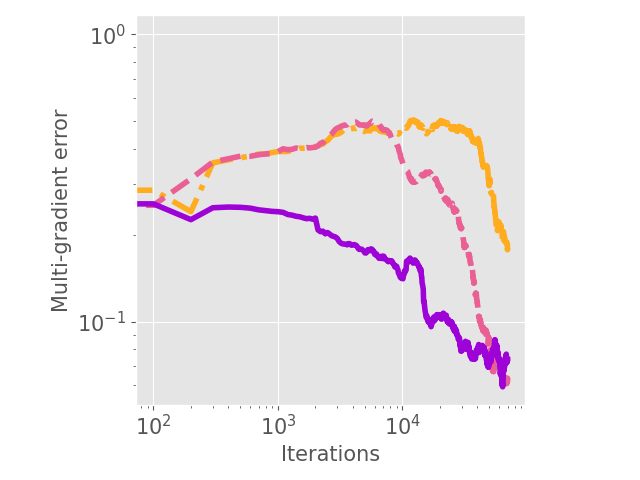}
         \label{fig:err-norm-traj1}
     \end{subfigure}
     \begin{subfigure}[b]{0.32\textwidth}
         \centering
         \includegraphics[trim={1cm 0 2cm 0},clip, width=\textwidth]{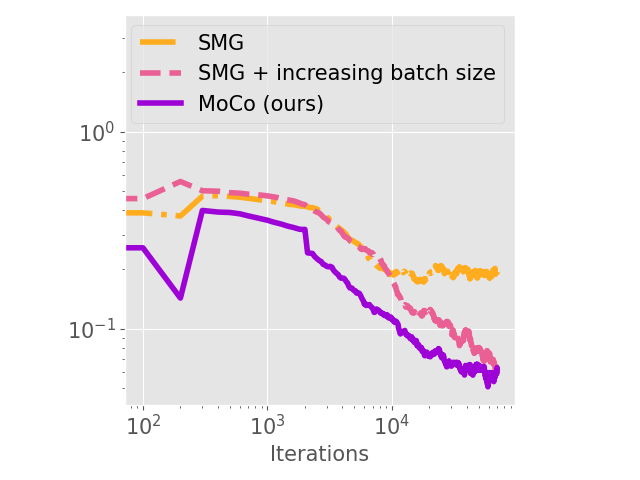}
         \label{fig:err-norm-traj2}
     \end{subfigure}
     \begin{subfigure}[b]{0.32\textwidth}
         \centering
         \includegraphics[trim={1cm 0 2cm 0},clip, width=\textwidth]{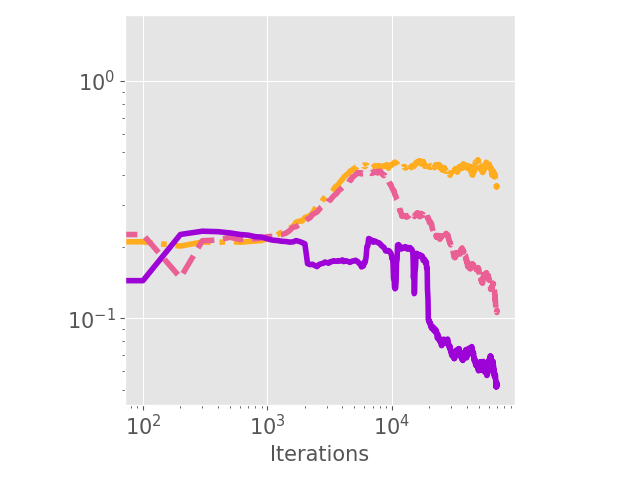}
         \label{fig:err-norm-traj3}
     \end{subfigure}\\
     \begin{subfigure}[b]{0.32\textwidth}
         \centering
         \includegraphics[trim={1cm 0 2cm 0},clip, width=\textwidth]{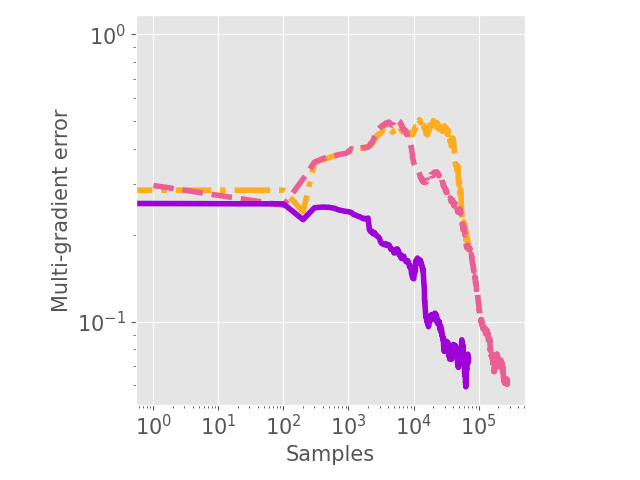}
         \caption{Trajectory 1}
         \label{fig:err-norm-traj1-sam}
     \end{subfigure}
     \begin{subfigure}[b]{0.32\textwidth}
         \centering
         \includegraphics[trim={1cm 0 2cm 0},clip, width=\textwidth]{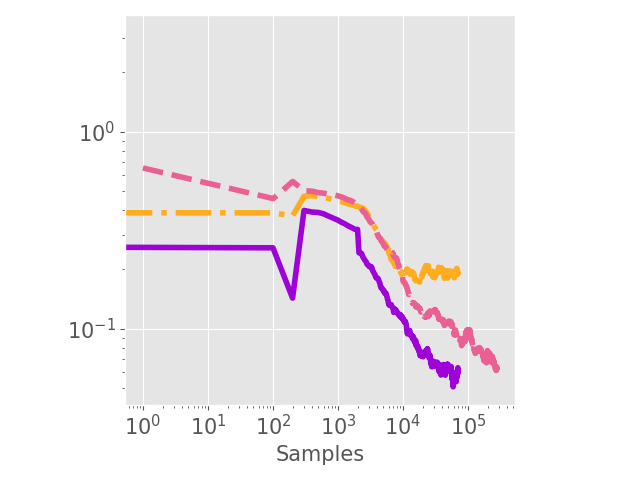}
         \caption{Trajectory 2}
         \label{fig:err-norm-traj2-sam}
     \end{subfigure}
     \begin{subfigure}[b]{0.32\textwidth}
         \centering
         \includegraphics[trim={1cm 0 2cm 0},clip, width=\textwidth]{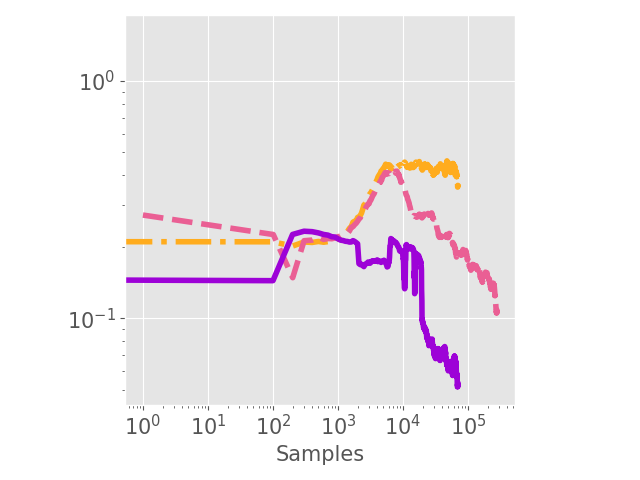}
         \caption{Trajectory 3}
         \label{fig:err-norm-traj3-sam}
     \end{subfigure}
        \caption{Comparison of multi-gradient error}
        \label{fig:mg-error-comp}
        \vspace{-0.2cm}
\end{figure}

\section{Details of experiments} \label{app:experiments}
In this section, we describe the omitted details of experiments in the main paper.

\subsection{Toy example}\label{app:toy}
In order to illustrate the advantages of our algorithm, we use a toy example introduced by \citep{liu2021conflict}. The example consists of optimizing two objectives $f_1(x)$ and $f_2(x)$ with $x=(x_1, x_2)^\top \in \mathbb{R}^2$, given by
\begin{align*}
    f_1(x) = p_1(x) q_1(x) + p_2(x) r_1(x)\\
    f_2(x) = p_1(x) q_2(x) + p_2(x) r_2(x)
\end{align*}
where we define
\begin{align*}
    q_1(x) &= \log \left( | 0.5 \max\left(-x_1 -7) - \tanh (-x_2)|, 0.000005\right)\right) + 6 ,\\
    q_2(x) &= \log \left( | 0.5 \max\left(-x_1 -7) - \tanh (-x_2) + 2|, 0.000005\right)\right) + 6 ,\\
    r_1(x) &= \left(\left( -x_1 + 7 \right)^2 + 0.1 \left(-x_1 - 8\right)^2\right)/10 - 20,\\
    r_2(x) &= \left(\left( -x_1 - 7 \right)^2 + 0.1 \left(-x_1 - 8\right)^2\right)/10 - 20,\\
    p_1(x) &= \max\left(\tanh\left( 0.5 x_2 \right), 0 \right),\\
    p_2(x) &= \max\left( \tanh \left(-0.5 x_2\right), 0\right).
\end{align*}

\textbf{Generation of Figure \ref{fig:toy-comp}.}
For generating the trajectories in Figure \ref{fig:toy-comp} we use 3 initializations $$x_0\in\{(-8.5, 7.5), (-8.5, 5), (10, -8)\},$$ and run each algorithm for 70000 iterations. For all the algorithms, we use the initial learning rate of $0.001$, exponentially decaying at a rate of $0.05$. In this example for  MoCo, we use $\beta_k=5/k^{0.5}$, where $k$ is the number of iterations.

\textbf{Generation of Figure \ref{fig:front-comp}.} For the comparison of MOO algorithms in objective space depicted in Figure \ref{fig:front-comp}, we use 5 initializations $$x_0\in\{(-8.5, 7.5), (-8.5, 5), (10, -8), (0, 0), (9, 9)\}.$$ The optimization configurations for each algorithm is similar to that of the aforementioned trajectory example, except with initial learning rate of $0.0025$. 

\textbf{Comparison with SMG with growing batch size.} 
For the comparison of the multi-gradient bias among SMG, SMG with increasingly large batch size, and MoCo, we use the norm of the error of the stochastic multi-gradient calculated using the three trajectories randomly initialized from $x_0\in\{(-8.5, 7.5), (-8.5, 5), (10, -8)\}$. For calculating the bias of the multi-gradient, we compute the multi-gradient using 10 sets of gradient samples at each point of the trajectory, take the average and record the norm of the difference between the computed average and true multi-gradient. All three methods are run for 70000 iterations, and follow the same optimization configuration used for Figure \ref{fig:front-comp}. For SMG with increasing batch size, we increase the number of samples in the minibatch used for estimating the gradient by one every 10000 iterations. We report the bias of the multi-gradient with respect to the number of iterations and also number of samples in Figure \ref{fig:mg-error-comp}. In the figure, Trajectories 1, 2, and 3 correspond to initializations $(-8.5, 7.5)$, $(-8.5, 5)$, and $(10, -8)$, respectively. It can be seen that our method performs comparable to SMG with increasing batch size, but with fewer samples. Furthermore, SMG  has non-decaying bias in some trajectories.



\subsection{Supervised learning}\label{app:sl}

\begin{figure}[t]
     \centering
     \begin{subfigure}[b]{0.32\textwidth}
         \centering
         \includegraphics[width=\textwidth]{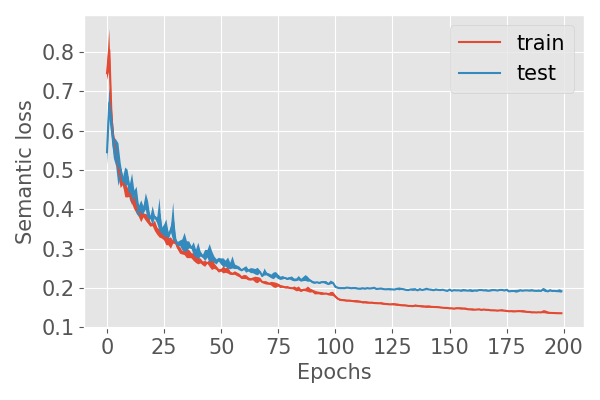}
         \caption{Semantic loss}
         \label{fig:semantic-loss-cityscapes}
     \end{subfigure}
     \begin{subfigure}[b]{0.32\textwidth}
         \centering
         \includegraphics[width=\textwidth]{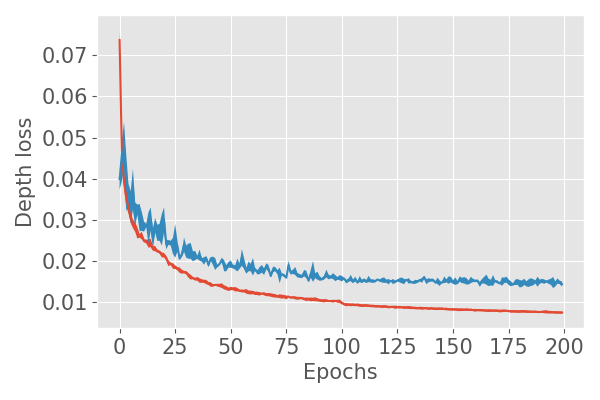}
         \caption{Depth loss}
         \label{fig:depth-loss-cityscapes}
     \end{subfigure}
        \caption{Training and test loss for the Cityscapes tasks}
        \label{fig:loss-cityscapes}
        \vspace{-0.2cm}
\end{figure}

\begin{figure}[t]
     \centering
     \begin{subfigure}[b]{0.32\textwidth}
         \centering
         \includegraphics[width=\textwidth]{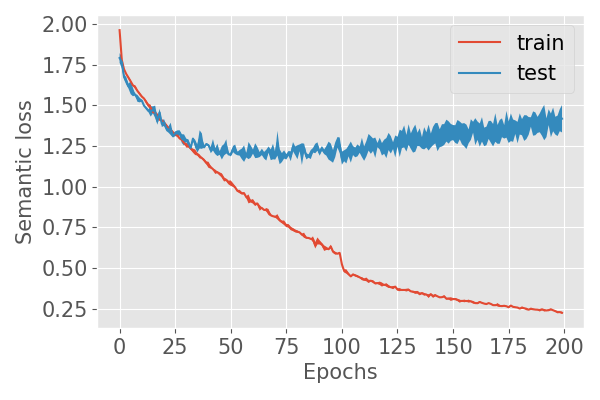}
         \caption{Semantic loss}
         \label{fig:semantic-loss-nyuv2}
     \end{subfigure}
     \begin{subfigure}[b]{0.32\textwidth}
         \centering
         \includegraphics[width=\textwidth]{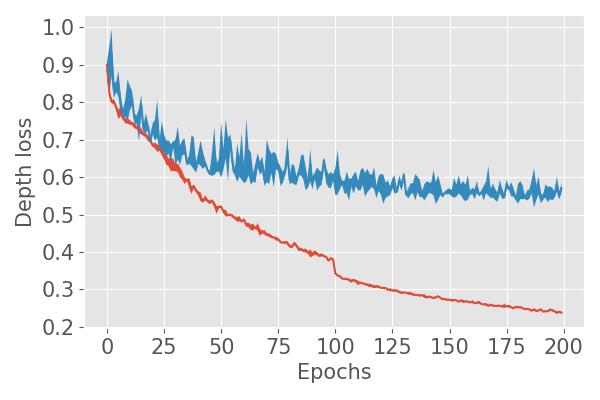}
         \caption{Depth loss}
         \label{fig:depth-loss-nyuv2}
     \end{subfigure}
     \begin{subfigure}[b]{0.32\textwidth}
         \centering
         \includegraphics[width=\textwidth]{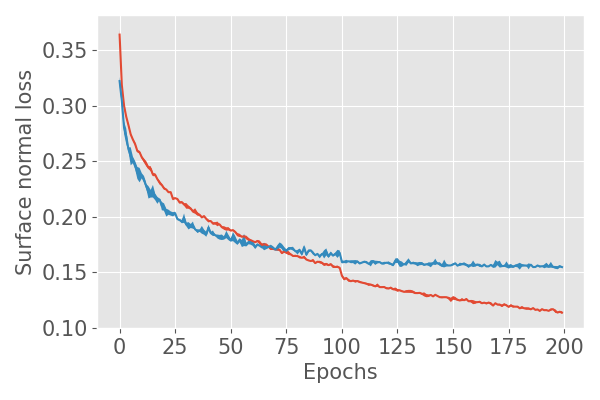}
         \caption{Surface normal loss}
         \label{fig:surf-norm-loss-nyuv2}
     \end{subfigure}
        \caption{Training and test loss for the NYU-v2 tasks}
        \label{fig:loss-nyu-v2}
        \vspace{-0.2cm}
\end{figure}

In this section we provide additional details and experiments on Cityscapes and NYU-v2 datasets, and also provide experiment results on two additional datasets Office-31 and Office-home.  Each experiment consists of solving multiple supervised learning problems related to each dataset. We consider each such task as an objective, which then can be simultaneously optimized by a gradient based MOO algorithm. We describe details on formulating each task as an objective next.

\textbf{Problem formulation.} 
First we look into the problem formulation of NYU-v2 and Cityscapes experiments.
For NYU-v2 dataset, 3 tasks are involved: pixel-wise 13-class classification, pixel-wise depth estimation, and pixel-wise surface normal estimation. In Cityscapes experiments, there are 2 tasks involved: pixel-wise 7-class classification and pixel-wise depth estimation. The following problem formulation applies for both NYU-v2 and Cityscapes tasks, excpet for the surface normal estimation task which only relates to NYU-v2 dataset. Let the set of images be $\{U\}_{i=1}^N$, pixel-wise class labels be $\{T_1\}_{i=1}^N$, pixel-wise depth values be $\{T_2\}_{i=1}^N$, and pixel-wise surface normal values be $\{T_3\}_{i=1}^N$, where $N$ is the number of training data samples. Let $x$ be the model that we train to perform all the tasks simultaneously. Let the image dimension be $P\times Q$. We will use $T_m$ for ground truth and $\hat{T}_m$ for the corresponding prediction by the model $x$, where $m\in[M]$ and $M$ is the number of tasks. We can now formulate the corresponding objective for each task, as given in \citep{liu2019end}. The objective for pixel-wise classification is pixel-wise cross-entropy, which is given as
\begin{align*}
    f_1(x) = -\frac{1}{NPQ}\sum_{i,p,q}T_{1,i}(p, q)\log \hat{T}_{1,i}(p, q) 
\end{align*}
where $i\in [N]$, $p\in [P]$, and $q\in [Q]$. Similarly, we can have the objectives for pixel-wise depth estimation and surface normal estimation, respectively, as
\begin{align*}
    f_2(x)  = \frac{1}{NPQ}\sum_{i,p,q} | T_{1,i}(p, q) - \hat{T}_{1,i}(p, q) | \quad \text{and}~~~
    f_3(x)  = \frac{1}{NPQ}\sum_{i,p,q} T_{1,i}(p, q) \cdot \hat{T}_{1,i}(p, q),
\end{align*}
where $\cdot$ is the elementwise dot product. With these objectives, we can formulate the problem \eqref{eq:prob-form} for Cityscapes and NYU-v2 tasks, with $f_3$ only used in the latter.

Similar to NYU-v2 and Cityscapes experiments, we can also formulate the supervised learning tasks on Office-31 and Office-home MTL as an instance of problem \eqref{eq:prob-form}. 

 \begin{table}[tb]
\small
    \centering
    \setlength{\tabcolsep}{1.0em} 
{\renewcommand{\arraystretch}{1.3}
    \begin{tabular}{c c c c c c}
    \hline
         \multirow{3}{*}{Method} & \multicolumn{2}{c}{Segmentation} & \multicolumn{2}{c}{Depth} & \multirow{3}{*}{$\Delta m\% \downarrow$}  \\ \cline{2-3}\cline{4-5}
         & \multicolumn{2}{c}{(Higher Better)} & \multicolumn{2}{c}{(Lower Better)} &\\
         & mIoU & Pix Acc & Abs Err & Rel Err & \\ \hline\hline
         Independent &            74.01       & 93.16             & 0.0125 & 27.77 & -  \\ \hline
         PCGrad \citep{yu2020gradient} &              75.13       & 93.48          & 0.0154 & 42.07 &  18.29\\
         CAGrad \citep{liu2021conflict} &             75.16       & 93.48          & 0.0141 & 37.60 &     11.64\\ \hline
         \textbf{MoCo (ours)} & 75.42 & 93.55 & 0.0149 & 34.19 &  9.90\\ 
         \textbf{PCGrad + MoCo} &              75.49       & 93.62          & 0.0146 & 46.07 &  20.07\\
         \textbf{CAGrad + MoCo} &             75.07       & 93.39          & 0.0137 & 36.78 &     10.12\\ \hline
    \end{tabular}}
    \vspace{5 pt}
    \caption{MoCo with existing gradient manipulation MTL algorithms for Cityscapes dataset tasks. Results are averaged over 3 independent runs.}
    \label{tab:cityscape-moco-plus-results}
            \vspace{-0.2cm}
\end{table}

\textbf{Cityscapes dataset.}
For implementing evaluation Cityscapes dataset, we follow the experiment set up used in \citep{liu2021conflict}. All the MTL algorithms considered are trained using a SegNet \citep{badrinarayanan2017segnet} model with attention mechanism MTAN \citep{liu2019end} applied on top of it for different tasks. All the MTL methods in comparison are trained for 200 epochs, using a batch size of  8. We use Adam as the optimizer with a learning rate of 0.0001 for the first 100 epochs, and with a learning rate of 0.00005 for the rest of the epochs. Following \citep{liu2021conflict} for each method in comparison we report the average test performance of the model over last 10 epochs, averaged over 3 seeds. For implementing MoCo in Cityscapes dataset we use $\beta_k=0.05/k^{0.5}$, $\gamma_k=0.1/k^{0.5}$, where $k$ is the iteration number. For the projection to simplex in the $\lambda_k$ update, we use a softmax function. The training and test loss curves for semantic segmentation and depth estimation are shown in Figures \ref{fig:semantic-loss-cityscapes} and \ref{fig:depth-loss-cityscapes} respectively.

\begin{table}[tb]
\vspace{-0.2cm}
\small
    \centering
    \setlength{\tabcolsep}{0.3em} 
{\renewcommand{\arraystretch}{1.4}
    \begin{tabular}{c c c c c c c c c c c}
    \hline
         \multirow{3}{*}{Method} & \multicolumn{2}{c}{Segmetation} & \multicolumn{2}{c}{Depth} & \multicolumn{5}{c}{Surface Normal} & \multirow{3}{*}{$\Delta m \% \downarrow$} \\ \cline{2-3}\cline{4-5}\cline{6-10}
         & \multicolumn{2}{c}{(Higher Better)} & \multicolumn{2}{c}{(Lower Better)} & \multicolumn{2}{c}{\makecell{Angle Distance \\ (Lower Better)}} & \multicolumn{3}{c}{\makecell{Within $t^\circ$ \\ (Higher better)}}\\
         & mIoU & Pix Acc & Abs Err & Rel Err & Mean & Median & 11.25 & 22.5 & 30\\ \hline\hline
         Independent            & 38.30 & 63.76 & 0.6754 & 0.2780 & 25.01 & 19.21 & 30.14 & 57.20 & 69.15 & - \\ \hline
         PCGrad \citep{yu2020gradient}           & 38.06 & 64.64 & 0.5550 & 0.2325 & 27.41 & 22.80 & 23.86 & 49.83 & 63.14 & 3.97 \\
         CAGrad \citep{liu2021conflict}     & 39.79 & 65.49 & 0.5486 & 0.2250 & 26.31 & 21.58 & 25.61 & 52.36 & 65.58 & 0.20\\\hline
         \textbf{MoCo (ours)}     & 40.30  & 66.07  & 0.5575  & 0.2135 & 26.67  & 21.83  & 25.61  & 51.78 & 64.85 &  0.16\\
         \textbf{PCGrad + MoCo}  & 38.80 & 65.02 & 0.5492 & 0.2326 & 27.39 & 22.75 & 23.64 & 49.89 & 63.21 & 3.62 \\
         \textbf{CAGrad + MoCo}  & 39.58 & 65.49 & 0.5535 & 0.2292 & 25.97 & 20.86 & 26.84 & 53.79 & 66.65 & -0.97\\\hline
    \end{tabular}}
    \vspace{5 pt}
    \caption{MoCo with existing gradient manipulation MTL algorithms for NYU-v2 dataset tasks. Results are averaged over 3 independent runs. }
    \label{tab:nyu-v2-moco-plus-results}
            \vspace{-0.4cm}
\end{table}

\textbf{NYU-v2 dataset.}
For NYU-v2 dataset, we follow the same setup as Cityscapes dataset, except with a batch size of 2. For implementing MoCo in NYU-v2 experiments, we use $\beta_k=0.99$, $\gamma_k=0.1$ with gradient normalization followed by weighting each gradient with corresponding task loss. This normalization was applied to avoid biasing towards one task, as can be seen is the case for MGDA. 
For the projection to simplex in the $\lambda_k$ update in MoCo, we apply softmax function to the update, to improve computational efficiency. The training and test loss curves for semantic segmentation, depth estimation, and surface normal estimation for NYU-v2 dataset are shown in Figures \ref{fig:semantic-loss-nyuv2},  \blue{\ref{fig:depth-loss-nyuv2}}, and \blue{\ref{fig:surf-norm-loss-nyuv2}} respectively. It can be seen that the model start to slightly overfit to the training data set with respect to the semantic loss after the 100th epoch. However this did not significantly harm the test performance in terms of accuracy compared to the other algorithms.

\begin{table}[th]
\small
    \centering
    \setlength{\tabcolsep}{1.0em} 
{\renewcommand{\arraystretch}{1.3}
    \begin{tabular}{c c c c c}
    \hline
         \multirow{2}{*}{Method} & \multicolumn{3}{c}{Domain} & \multirow{2}{*}{$\Delta m\%$}\\ \cline{2-4}
         & Amazon & DSLR & Webcam &\\
         \hline
         Mean & 84.22 & 94.81 & 97.04 & -\\
         MGDA & 79.60 & 96.45 & 97.96 & 0.94\\
         PCGrad & 84.10 & 95.08 & 96.30 & 0.21\\
         GradDrop  & \textbf{84.73} & 96.17 & 96.85 & -0.61\\
         CAGrad  & 84.22 & 94.26 & 97.41 & 0.07\\
         \textbf{MoCo (ours)} & 84.33 & \textbf{97.54} & \textbf{98.33} & \textbf{-1.45}\\ \hline
    \end{tabular}}
    \vspace{5 pt}
    \caption{Results on Office-31 dataset. We show the 31-class classification results over 3 domains on Office-31 data set. The results are obtained from the epoch with the best validation performance.}
    \label{tab:office-31-results}
                \vspace{-0.4cm}
\end{table}

\begin{table}[t]
\small
    \centering
    \setlength{\tabcolsep}{1.0em} 
{\renewcommand{\arraystretch}{1.3}
    \begin{tabular}{c c c c c c}
    \hline
         \multirow{2}{*}{Method} & \multicolumn{4}{c}{Domain} & \multirow{2}{*}{$\Delta m\%$}\\ \cline{2-5}
         & Art & Clipart & Product & Real-World\\
         \hline
         Mean & \textbf{63.88} & 77.90 & 89.55 &  79.39 & -\\
         MGDA & 63.63 & 73.78 & 90.18 & \textbf{79.82} & 1.11\\
         PCGrad & 63.06 & 77.46 & 89.09 & 78.70 & 0.81\\
         GradDrop & 63.82 & 78.22 & 89.19 & 78.88 & 0.18\\
         CAGrad & 63.06 & 77.03 & 89.62 & 79.53 & 0.54\\
         \textbf{MoCo (ours)} & 63.38 & \textbf{79.41} & \textbf{90.25} & 78.70 & \textbf{-0.27}\\ \hline
    \end{tabular}}
    \vspace{5 pt}
    \caption{Multi-task learning results on Office-Home dataset. We show the 65-class classification results over 4 domains on Office-Home data set. For each method the results are obtained from the epoch with the best validation performance.}
    \label{tab:office-home-results}
                \vspace{-0.8cm}
\end{table}

\textbf{MoCo with existing MTL algorithms} We apply the gradient correction introduced in MoCo on top of existing MTL algorithms to further improve the performance. Specifically, we apply the gradient correction of MoCo for PCGrad and CAGrad on Cityscapes and NYU-v2. 
For the gradient correction (update step \eqref{eq:y update}) in PCGrad we use $\beta_k=0.99$ for both Cityscapes and NYU-v2 datasets, and for that in CAGrad we use $\beta_k=0.99$ for Cityscapes dataset and $\beta_k=0.99/k^{0.5}$ for NYU-v2 dataset. The results are shown in Tables \ref{tab:cityscape-moco-plus-results} and \ref{tab:nyu-v2-moco-plus-results} for Cityscapes and NYU-v2 datasets, respectively. We restate the results for independent task performance and original MTL algorithm performance for reference. It can be seen that the gradient correction improves the performance of the algorithm which only use stochastic gradients. For PCGrad where no explicit convex combination coefficient computation for gradients is involved, there is an improvement of $\Delta m\%$ for NYU-v2 by $0.35\%$. For Cityscapes, it can be seen a slight degradation in terms of $\Delta m\%$, in exchange for improvement in 3 out of 4 performance metrics. This can be expected as PCGrad does not explicitly control the converging point to be closer to the average loss. For CAGrad which explicitly computes dynamic convex combination coefficients for gradients using stochastic gradients such that it converges closer to a point that perform well in terms of average task loss, there is an improvement of $\Delta m\%$ for Cityscapes by $1.52\%$ and that for NYU-v2 by $1.17\%$. This suggests that incorporating the gradient correction of MoCo in existing gradient based MTL algorithms also boosts their performance. 

In addition to the experiments described above, we demonstrate the performance of MoCo in comparison with other MTL algorithms using Office-31 \citep{saenko2010adapting} and Office-home\citep{venkateswara2017deep} datasets. Both of these datasets consist of images of several classes belonging to different domains. We use the method "Mean" as the baseline for $\Delta m\%$, instead of independent task performance. The Mean baseline is the method where the average of task losses on each domain is used as the single objective optimization problem. For reporting per domain performance for all the methods compared in Office-13 and Office-home experiments, the test performance at the epoch with highest average validation accuracy (across domains) is used for each independent run. This performance measure is then averaged over three independent runs.

\textbf{Office-31 dataset.}
 The dataset  consists of three classification tasks on 4,110 images collected from three domains: Amazon, DSLR, and Webcam, where each domain has 31 object categories. 
 For implementing experiments using Office-home and Office-31 datasets,  we use the experiment setup and implementation given by LibMTL framework \citep{lin2022libmtl}. The MTL algorithms are implemented using hard parameter sharing architecture, with ResNet18 backbone. As per the implementation in \citep{lin2022libmtl}, 60\% of the total dataset is used for training, 20\% for validation, and the rest 20\% for testing. All methods in comparison are run for 100 epochs. For MoCo implementation, we use $\beta_k=0.5/k^{0.5}$ and $\gamma_k=0.1/k^{0.5}$. We report the test performance of best performing model based on validation accuracy after each epoch, averaged over 3 seeds. The results are given in Table \ref{tab:office-31-results}. It can be seen that MoCo significantly outperforms other methods in most taks, and also in terms of $\Delta m\%$.

\textbf{Office-home dataset.} 
This dataset  consists of four classification tasks over 15,500 labeled images on four domains; Art: paintings, sketches and/or artistic depictions, Clipart: clipart images, Product: images without background and Real-World: regular images captured with a camera. Each domain has 65 object categories. 
We follow the same experiment setup as Office-31, and for MoCo implementation, we use $\beta_k=0.5/k^{0.5}$ and $\gamma_k=0.1/k^{0.5}$. The results are given in Table \ref{tab:office-home-results}. It can be seen that MoCo significantly outperforms other methods in most taks, and also in terms of $\Delta m\%$.






In the Table \ref{tab:hp-summary} we summarize the hyper-parameters choices used for MoCo in each of the experiments we have reported in this paper.

\begin{table}[th]
    \centering
    \setlength{\tabcolsep}{1.0em} 
{\renewcommand{\arraystretch}{1.3}
    \begin{tabular}{|c| c | c | c| c | c|}
    \hline
         & Cityscapes & NYU-v2 & Office-31 & Office-home & MT10\\
         \hline
         optimizer of $x_k$ & Adam & Adam & Adam & Adam & Adam\\
         $x_k$ learning rate ($\alpha_k$) & 
         0.0001&
              0.0001
         &0.0001 & 0.0001 & 0.0003\\
         $Y_k$ learning rate ($\beta_k$) & $0.05/k^{0.5}$ & $0.99$ & $0.5/k^{0.5}$ & $0.5/k^{0.5}$ & 0.99\\
         $\lambda_k$ learning rate ($\gamma_k$) & $0.1/k^{0.5}$ & $0.1$ & $0.1/k^{0.5}$ & $0.1/k^{0.5}$ & 10\\
         batch size  & 8 & 2 & 64 & 64 & 1280\\
         training epochs & 200 & 200 & 100 & 100 & 2 million steps\\ \hline
    \end{tabular}
    \vspace{5 pt}
    \caption{Summary of hyper-parameter choices for MoCo in each experiment}
    \label{tab:hp-summary}
                \vspace{-0.4cm}}
\end{table}

\end{document}